\documentclass[12pt]{article}

\usepackage{fullpage}
\usepackage[utf8]{inputenc} 
\usepackage[T1]{fontenc}    
\usepackage{hyperref}       
\usepackage{url}            
\usepackage{booktabs}       
\usepackage{amsfonts}       
\usepackage{nicefrac}       
\usepackage{microtype}      
\usepackage{microtype}
\usepackage{graphicx}
\usepackage{float}
\usepackage{soul}
\usepackage{caption}
\usepackage{bbding}
\usepackage{float}
\usepackage{sidecap}
\usepackage{enumitem}
\usepackage{amsmath,amssymb,amsthm}
\usepackage{color}
\usepackage{algorithm,algorithmic}
\usepackage{amsfonts}
\usepackage{subfigure}

\usepackage{hyperref}
\usepackage{amsmath}
\usepackage{booktabs}
\usepackage{hyperref}
\usepackage[left=0.75in, right=0.75in, top=1in, bottom=1in]{geometry}

\title{Theoretical Guarantees of Fictitious Discount Algorithms for Episodic Reinforcement Learning and Global Convergence of  Policy Gradient Methods}

\author{Xin Guo 
\thanks{University of California, Berkeley. 
\textbf{Email:}  \texttt{xinguo@berkeley.edu}} 
  \thanks{Amazon.com. \textbf{Email:} \texttt{xnguo@amazon.com}}
  \footnotemark[5]
    \and
Anran Hu 
\thanks{University of California,  Berkeley. 
\textbf{Email:} \texttt{anran\_hu@berkeley.edu}}
\and
Junzi Zhang
\thanks{Amazon.com. \textbf{Email:} \texttt{junziz@amazon.com}} 
  \thanks{Work done prior to joining or outside of Amazon.}
}

\newcommand{\BEAS}{\begin{eqnarray*}}
\newcommand{\EEAS}{\end{eqnarray*}}
\newcommand{\BEQ}{\begin{equation}}
\newcommand{\EEQ}{\end{equation}}
\newcommand{\BIT}{\begin{itemize}}
\newcommand{\EIT}{\end{itemize}}

\newcommand{\eg}{{\it e.g.}}
\newcommand{\ie}{{\it i.e.}}
\newcommand{\cf}{{\it cf. }}

\newcommand{\reals}{{\mbox{\bf R}}}

\newcommand{\Expect}{\mathbf{E}}
\newcommand{\prob}{\mathbf{Prob}}
\newcommand{\var}{\mathop{\bf var}}

\newcommand{\argmin}{\mathop{\rm argmin}}

\newtheorem{theorem}{Theorem}
\newtheorem{remark}{Remark}
\newtheorem{assumption}{Assumption}
\newtheorem{lemma}[theorem]{Lemma}
\newtheorem{corollary}[theorem]{Corollary}
\newtheorem{proposition}[theorem]{Proposition}

{\begin{quote}}{\end{quote}}

\makeatletter
\long\def\@makecaption#1#2{
   \vskip 9pt
   \begin{small}
   \setbox\@tempboxa\hbox{{\bf #1:} #2}
   \ifdim \wd\@tempboxa > 5.5in
        \begin{center}
        \begin{minipage}[t]{5.5in}
        \addtolength{\baselineskip}{-0.95pt}
        {\bf #1:} #2 \par
        \addtolength{\baselineskip}{0.95pt}
        \end{minipage}
        \end{center}
   \else
	\hbox to\hsize{\hfil\box\@tempboxa\hfil}
   \fi
   \end{small}\par
}
\makeatother

\newcounter{oursection}

\newcounter{lecture}

\begin{document}

\maketitle

\begin{abstract}
When designing algorithms for finite-time-horizon episodic reinforcement learning problems, a common approach is to introduce a fictitious discount factor and use stationary policies for approximations. Empirically, it has been shown that the fictitious discount factor  helps reduce variance, and stationary policies serve to save the per-iteration computational cost. Theoretically, however, there is no existing work on convergence analysis for algorithms with this fictitious discount recipe. This paper takes the first step towards analyzing these algorithms. It focuses on two vanilla policy gradient (VPG) variants: the first being  a widely used variant with discounted advantage estimations (DAE), the second with an additional fictitious discount factor in the score functions of the policy gradient estimators. 
 Non-asymptotic convergence guarantees are established for both algorithms, and  the additional discount factor is shown to reduce the bias introduced in DAE and thus improve the algorithm convergence asymptotically. A key ingredient of our analysis is to 
 connect  three settings of Markov decision processes (MDPs): the finite-time-horizon, the average reward  and the discounted settings. To our best knowledge, this is the first theoretical guarantee on fictitious discount algorithms for the episodic reinforcement learning of finite-time-horizon MDPs, which also leads to the (first) global convergence  of policy gradient methods for finite-time-horizon episodic reinforcement learning.
\end{abstract}

\section{Introduction}
This paper studies episodic reinforcement learning with each episode consisting of a finite-time-horizon Markov decision process (MDP). For such finite-time-horizon episodic reinforcement learning problems, a popular heuristic approach is to introduce a fictitious discount factor and use stationary policies when designing algorithms; see for instance, the renowned DQN \cite{DQN_original}, DDPG \cite{DDPG}, and recent works of \cite{franccois2015discount,xu2018meta,RND,hessel2018rainbow,fedus2019hyperbolic,tessler2020reward,amit2020discount}. 

Empirically, it has been shown that discount factors serve to reduce variance \cite{thomas2014bias, haarnoja2017reinforcement}, and stationary policies help save per-iteration computational costs. Theoretically, fictitious discount algorithms designed for \emph{average reward} MDPs  have been analyzed \cite{marbach1998simulation, marbach2001simulation} and the \emph{asymptotic} local convergence\footnote{In this paper,  ``local convergence'' indicates convergence to stationary points of value functions, and  ``global convergence'' means convergence in terms of the value function sub-optimality gaps.} has been established \cite{marbach2003approximate}. 

It remains open, however, to establish the non-asymptotic global convergence for this  fictitious-discount-factor approach in the {\it finite-time-horizon} framework. 
The major challenges are to characterize the \emph{bias} introduced by the discount factor, and to close the gap between the \emph{non-stationary} optimal policies for finite-time-horizon MDPs and the stationary algorithm policies. 

This paper takes the first steps towards rigorously analyzing the global and non-asymptotic convergence  of  fictitious discount algorithms for finite-time-horizon episodic reinforcement learning. It focuses on the convergence analysis of two concrete algorithms in the context of policy gradient methods.
The first one is a widely used variant of the vanilla policy gradient (VPG) method with discounted advantage estimations (DAE). This variant was originally proposed for average reward problems \cite{marbach1998simulation, baxter1999direct,baxter2001infinite,marbach2001simulation}, later extended to episodic deep reinforcement learning setting \cite{schulman2015high} and implemented in popular solvers such as Spinning Up \cite{pg_impl}. 
 The second one is a new doubly discounted variant of VPG, with the introduction of an additional fictitious discount factor in the score functions of the policy gradient estimators. 
 This additional discount factor is shown to help reduce the bias in DAE and thus improve asymptotically the algorithm convergence.

 \paragraph{Our approach.} 
\hspace{-0.306cm} There are three main ingredients in our analysis. 
 The first is establishing  quantitative connections among three settings of MDPs: the finite-time-horizon, the average award, and the discounted settings (\cf \S\ref{prelim}). These relations enable us to connect the finite-time-horizon sub-optimality gap with the average reward (\cf Theorem \ref{main_res_finite_average}) and the discounted (\cf Theorem \ref{main_res_finite_discount}) ones.
 The second is utilizing the convergence property of value iteration algorithms  to analyze the gap between the stationary policies of the average reward MDPs 
 and the non-stationary optimal policies of the finite-time-horizon MDPs (\cf Lemma \ref{VH-eta}).
 The third one is deriving the gradient domination (\cf Lemma \ref{grad_dom_average}) and Lipschitz gradient (\cf Lemma \ref{smooth_pg_ave}) properties for average reward MDPs, which is critical to obtain the sub-optimality of algorithm policies for the average reward problem (\cf Theorem \ref{sample_complexity_eta_thm}). 
 
  \paragraph{Contributions.} \hspace{-0.306cm} The contributions of this paper are two-fold: 
 \begin{itemize}
  \item It establishes the first (and non-asymptotic) connections between (a) the sub-optimality gap in finite-time-horizon MDPs and (b) the sub-optimality gaps in the average reward and the discounted reformulations (\cf Theorems \ref{main_res_finite_average} and \ref{main_res_finite_discount}). 
 \item It obtains, for the first time, theoretical guarantees on fictitious discount algorithms for the episodic reinforcement learning of finite-time-horizon MDPs (\cf Theorems \ref{main_conv_spin-up} and \ref{discounted_reinforce_thm}). The convergence is global, and not asymptotic. Moreover, it demonstrates explicit dependencies on both the  time horizon  and the fictitious discount factor. 
The analysis in this paper leads to the first global convergence of policy gradient methods for finite-time-horizon episodic reinforcement learning.
 \end{itemize}

 \paragraph{Related work.} 
 
 \hspace{-0.306cm} Since the seminal work of D. Blackwell \cite{blackwell1962discrete}, earlier works on the relationship among different settings of MDPs have been focusing on the discounted and average reward settings 
 \cite{hordijk2002blackwell,lasserre1988conditions,kakade2001optimizing,lewis2002bias,mahadevan1996sensitive,schneckenreither2020average}. 
 In contrast, our focus is on the remaining two relations, 
 namely (i) the connection between the finite-time-horizon and the discounted problems and (ii) the connection between the finite-time-horizon and the average reward problems. 
 
 Theoretical study on policy gradient methods started with the asymptotic local convergence \cite{sutton2000policy, actor-critic,marbach2001simulation}. Later, non-asymptotic rate of such local convergence has been established in a series of works \cite{SVRPG, xu2019sample}. 
Recently, more attention has been shifted to the global convergence of policy gradient methods. However, the majority of these results have been on the discounted settings \cite{zhang2019global,bhandari2019global, agarwal2019reinforcement,wang2019neural, shani2019adaptive, softmax_PG,cen2020fast, zhang2020variational}. Recent progress has been made on a particular class of   finite-time-horizon MDPs, i.e., linear quadratic finite-time-horizon MDPs  and their variants  \cite{hambly2020policy}  \cite{zhang2021derivative},  \cite{hambly2021policy}. This paper, instead, studies global convergence of policy gradient methods for finite-time-horizon, finite-state-action MDPs with {\it general dynamics and rewards}. 

 \paragraph{Outline.}   \S\ref{prelim} introduces three settings of MDPs and their mutual connections. \S\ref{dae_reinforce} introduces DAE REINFORCE  and establishes its global sub-optimality guarantee. A doubly discounted 
 variant is then proposed in \S\ref{doubly_discounted} with its global convergence analysis,  showing the benefits of the additional discount factor. \S\ref{extension} concludes.
 
 \section{Problem setup and preliminaries}\label{prelim}
\subsection{Problem Setup}\label{prob_intro}
Consider a Markov decision process $\mathcal{M}$ with a finite state space $\mathcal{S}=\{1,\dots,S\}$,  
a finite action space $\mathcal{A}=\{1,\dots,A\}$,  
a transition probability $p(s'|s,a)$ for the probability
of transitioning from state $s$ to state $s'$ when taking action $a$, and a reward function $r(s,a)$ denoting the (deterministic) 
instantaneous
reward for taking action $a$ in state $s$. Here, the initial state is assumed to follow a distribution $\rho\in\mathcal{P}(\mathcal{S})$, where $\mathcal{P}(\mathcal{S})\subseteq \reals^{|\mathcal{S}|}$ denotes the set of  probability measures on over the set $\mathcal{S}$. Denote $R_{\max}$ the maximum reward such that $R_{\max}=\max_{s\in\mathcal{S},a\in\mathcal{A}}\,|r(s,a)|$.

The focus of this paper is the finite-time-horizon MDP. Given a finite time horizon $H\geq 1$, decisions are made in the duration of timestamps from $h=0$ to $h=H-1$. 
This duration is also referred to as an ``episode''. Such a horizon can either be naturally defined by the expiration time (\eg, the length of a video game) or manually specified by the decision maker (\eg, the length of affordable decision period). A (randomized) policy $\pi:\mathcal{S}\rightarrow\mathcal{P}(\mathcal{A})$ is a mapping from the state space to a distribution over the action space. For notational simplicity, we use $\pi(a|s)$ to denote the $a$-th entry of $\pi(s)$, \ie, the probability of taking action $a$ at state $s$ under a policy $\pi$. 
Then for any (randomized) policy sequence ${\boldsymbol \pi}^H=\{\pi_h\}_{h=0}^{H-1}$, the performance metric $V^H({\boldsymbol\pi}^H)$ is the mean reward collected over the finite horizon episode of length $H$, \ie,
\BEQ\label{finite_horizon_reward}
V^H({\boldsymbol \pi}^H)=\dfrac{1}{H}\Expect\sum\nolimits_{h=0}^{H-1}r(s_h,a_h),
\EEQ
where 
$s_0\sim \rho$, $a_h\sim \pi_h(s_h)$ and $s_{h+1}\sim p(\cdot|s_h,a_h)$ for $h=0,\dots,H-2$. The finite-time-horizon problem is the following optimization problem:
\BEQ\label{finite_horizon_opt}
\text{maximize}_{{\boldsymbol \pi}^H=\{\pi_0,\dots,\pi_{H-1}\}} \,V^H({\boldsymbol \pi}^H).
\EEQ
Note that the optimal policy sequence ${\boldsymbol\pi}^{H,\star}=\{\pi_h^{H,\star}\}_{h=0}^{H-1}$ of problem \eqref{finite_horizon_opt} may be nonstationary, and we write $V^{H,\star}=V^H({\boldsymbol\pi}^{H,\star})$. When the policy sequence ${\boldsymbol\pi}^H=\{\pi\}_{h=0}^{H-1}$ is stationary, we will write it as $\pi$ for notational simplicity.  Here and below we use $P_{\pi}\in\reals^{S\times S}$ to denote the transition probability of the Markov chain induced by policy $\pi$, \ie, $P_{\pi}(s,s')=\sum_{a\in\mathcal{A}}p(s'|s,a)\pi(a|s)$.


Throughout this paper, we make the following assumption as in \cite{ortner2020regret}. Note that this assumption naturally holds when the transition probability $p$ is component-wisely positive. 
\begin{assumption}\label{finite-ergodic}
For any deterministic stationary policy $\pi$, 
the induced Markov chain with transition matrix $P_\pi$ is irreducible and aperiodic. 
\end{assumption}
With Assumption \ref{finite-ergodic}, we have the following proposition. 
\begin{proposition}\label{dobrushin_uniform}
Given Assumption \ref{finite-ergodic}, then there exist constants $C_{p,S,A}>1$ and $\alpha_{p,S,A}\in[0,1)$ that depend only on the transition probability model $p$, number of states $S$ and number of actions $A$ of the MDP $\mathcal{M}$, such that for any policy $\pi$ and $h\geq 0$, 
\BEQ\label{d_TV_alpha_uniform}
d_{\rm TV}(\rho P_{\pi}^h,\mu_{\pi})\leq C_{p,S,A}\alpha_{p,S,A}^{h},
\EEQ
where $\mu_{\pi}$ is the (unique) stationary distribution of the transition matrix $P_{\pi}$.  
\end{proposition}

The analysis of the above finite-time-horizon MDP will rely on two related MDPs:  the average reward problem and the discounted one, both of which have stationary optimal policies under Assumption \ref{finite-ergodic}. 


\paragraph{Discounted problem.} It is to consider an infinite horizon and solve for
\[
\text{maximize}_{\boldsymbol\pi=\{\pi_h\}_{h=0}^{\infty}}\,V^{\gamma}({\boldsymbol \pi})
\]
with
\[
V^{\gamma}({\boldsymbol \pi})=(1-\gamma)\Expect\sum\nolimits_{h=0}^{\infty}\gamma^hr(s_h,a_h),
\]
where $s_0\sim \rho$, $a_h\sim \pi_h(s_h)$ and $s_{h+1}\sim p(\cdot|s_h,a_h)$ for $h\geq 0$. Here $\gamma\in[0,1)$ is the discount factor, penalizing future rewards. 
It is well-known that for this discounted problem, there exists a stationary optimal policy sequence ${\boldsymbol \pi}^{\gamma,\star}=\{\pi_h^{\gamma,\star}\}_{h=0}^{\infty}$, where all $\pi_h^{\gamma,\star}=\pi^{\gamma,\star}$ ($h\geq 0$) are equal \cite{puterman2014markov}. Similarly, we denote $V^{\gamma,\star}=V^{\gamma}({\boldsymbol\pi}^{\gamma,\star})$.  Again, when the policy sequence $\boldsymbol\pi=\{\pi\}_{h=0}^{\infty}$ is stationary, we will write it as $\pi$ for notational simplicity. 

\paragraph{Average reward problem.} The infinite horizon average reward of a (stationary)  policy $\pi$ is defined as
\BEQ\label{eta-def}
\begin{split}
\eta(\pi)=&\lim_{H\rightarrow\infty}V^H(\pi)=\lim_{H\rightarrow\infty}\dfrac{1}{H}\Expect\sum_{h=0}^{H-1}r(s_h,a_h)=\sum_{s\in\mathcal{S},\,a\in\mathcal{A}}\mu_{\pi}(s)\pi(a|s)r(s,a),
\end{split}
\EEQ
where $\mu_{\pi}$ is defined in Proposition \ref{dobrushin_uniform}.  
The goal is to find $\pi$ that maximizes $\eta(\cdot)$. 
Note $\eta(\pi)$ is well-defined as the limit in \eqref{eta-def} is guaranteed to exist and be finite, and independent of the initial state distribution $\rho$ under Assumption \ref{finite-ergodic} \cite{puterman2014markov}. Since $|\eta(\pi)|\leq R_{\max}$ and the set of all (stationary) policies (viewed as a subset $\reals^{SA}$) is compact, the optimal (stationary) policy $\pi^\star$ (that maximizes $\eta(\cdot)$) exists and we denote the corresponding value function as $\eta^\star=\eta(\pi^\star)$. 

\subsection{Connections of finite-time-horizon with  discounted  and average reward problems}
Now we introduce our first set of main results, which characterize the connections within these three different MDP problems. 

The first result bounds the error between $V^\gamma(\pi)$ (for the discounted problem) and $V^H(\pi)$ (for the finite-time-horizon problem) under an arbitrary stationary policy $\pi$. 
\begin{lemma}\label{error_bd_arb} 
Given Assumption \ref{finite-ergodic}, then for any stationary policy $\pi$, 
\begin{equation}\label{VH-Vgamma-arb}
\begin{split}
|V^\gamma(\pi)-V^H(\pi)|\leq&\,2R_{\max}C_{p,S,A}\left(\frac{\gamma}{H(1-\gamma)}\alpha_{p,S,A}^H +\dfrac{\alpha_{p,S,A}+|H(1-\gamma)-1|}{(1-\alpha_{p,S,A})H}\right), 
\end{split}
\end{equation}
where $C_{p,S,A}>1$ and $\alpha_{p,S,A}\in[0,1)$ are the constants in Proposition \ref{dobrushin_uniform}, and 
depend only on the transition probability model $p$, the number of states $S$ and the number of actions $A$ of $\mathcal{M}$, the underlying MDP. 
\end{lemma}

The next lemma establishes a bound between $V^{\gamma}(\pi)$ (for the discounted problem) and $\eta(\pi)$ (for the average reward problem) under any stationary policy $\pi$. 
\begin{lemma}\label{Vgamma-eta-pi}
Given Assumption \ref{finite-ergodic}, then
\BEQ\label{Vgamma-eta-pi-bd}
|V^{\gamma}(\pi)-\eta(\pi)|\leq \dfrac{2(1-\gamma)R_{\max}C_{p,S,A}}{1-\alpha_{p,S,A}},
\EEQ
where the constants $C_{p,S,A}>1$ and $\alpha_{p,S,A}\in[0,1)$  are the same as in Lemma \ref{error_bd_arb}. 
\end{lemma}
Maximizing over $\pi$, then immediately from Lemma \ref{Vgamma-eta-pi}, we have
\begin{corollary}\label{Vgamma-eta}
Given Assumption \ref{finite-ergodic}, then
\BEQ\label{Vgamma-eta-star}
|V^{\gamma,\star}-\eta^\star|\leq \dfrac{2(1-\gamma)R_{\max}C_{p,S,A}}{1-\alpha_{p,S,A}},
\EEQ
where the constants $C_{p,S,A}>1$ and $\alpha_{p,S,A}\in[0,1)$ are the same as in Lemma \ref{error_bd_arb}. 
\end{corollary}

The following statement controls the gap between $V^{H}(\pi)$ (for the finite-time-horizon problem) and $\eta(\pi)$ (for the average reward problem) under any stationary policy $\pi$. 
\begin{lemma}\label{VH-eta-pi}
Given Assumption  \ref{finite-ergodic}, then 
\BEQ\label{VH-eta-pi-bd}
|V^H(\pi)-\eta(\pi)|\leq \frac{2R_{\max}C_{p,S,A}}{H(1-\alpha_{p,S,A})},
\EEQ
where the constants $C_{p,S,A}>1$ and $\alpha_{p,S,A}\in[0,1)$ are the same as in Lemma \ref{error_bd_arb}. 
\end{lemma}

And finally, the bound of 
the gap between the optimal value functions $V^{H,\star}$ (for the finite-time-horizon problem) and $\eta^\star$ (for the average reward problem) is as follows. 
\begin{lemma}\label{VH-eta}
Given Assumption \ref{finite-ergodic}, then 
\BEQ\label{VH-eta-star}
|V^{H,\star}-\eta^\star|\leq \dfrac{2R_{\max}D_{p,S,A}}{H}, 
\EEQ
where $D_{p,S,A}> 1$ is a constant that depends only on the transition probability model $p$, the number of states $S$ and the number of actions $A$ of the underlying MDP $\mathcal{M}$. 
\end{lemma}  
\begin{remark}
 Lemma \ref{VH-eta} cannot be directly implied by Lemma \ref{VH-eta-pi}. 
The key issue is that the optimal policy for the average reward value function $\eta(\cdot)$ is stationary, while the optimal policy for the finite-horizon value function $V^H(\cdot)$ may be non-stationary. To bridge this gap between stationary and non-stationary policies, we need  the convergence property of value iteration algorithms  (\cf Appendix \ref{proofs_subopt_lemmas}). 
\end{remark}

These properties show that the three different settings are closely related for a large horizon $H$, and are critical for the subsequent analyses. 

\subsection{Gradient properties}
In this section, we review the basics of policy gradient methods and state some useful 
properties of policy gradients in the average reward and the discounted settings. 

\paragraph{Policy gradient methods.} Policy gradient methods start by parametrizing the policy with parameter $\theta\in\Theta$, which we denote as $\pi_{\theta}$. Here $\Theta$ is the parameter space and the parametrization 
maps $\theta$ to a randomized policy $\pi_{\theta}:\mathcal{S}\rightarrow\mathcal{P}(\mathcal{A})$.  
The (vanilla) policy gradient (VPG) methods then proceed by performing stochastic gradient ascent on a (regularized) value function in the parameter space, namely, for each iteration $k$, $\theta^k$ is updated to $\theta^{k+1}$ with
\BEQ\label{vpg_prototype}
\theta^{k+1}=\theta^k+\alpha^kg_k.
\EEQ
Here $\theta^0$ is the initial parameter, $\alpha^k$ is the step-size, and $g_k$ is a (possibly biased) stochastic gradient estimator of a regularized value function. 

Throughout this paper, we will focus on the following regularized value function of the average reward problem: 
\[
 \bar{L}(\theta)=\eta(\pi_{\theta})+\Omega(\theta),
\]
and the regularized value function of the discounted problem:
\[
L^{\gamma}(\theta)=\frac{1}{1-\gamma}V^{\gamma}(\pi_{\theta})+\Omega(\theta).
\]
Here $\Omega:\Theta\rightarrow\reals$ is a regularization term that serves to improve the convergence \cite{zhao2016regularized, mnih2016asynchronous, henkel2018regularization}. 



Below we specify additional assumptions about the problem setting. Note that the same set of assumptions have been made in \cite{agarwal2019theory, zhang2020sample}.
\begin{assumption}(Setting)\label{setting}
\begin{itemize}
\item The policy is a soft-max policy parameterization, \ie, $\pi_{\theta}(a|s)=\frac{\exp(\theta_{s,a})}{\sum_{a'\in\mathcal{A}}\exp(\theta_{s,a'})}$, with the parameter space being $\Theta=\reals^{SA}$. 
\item The regularization term is (with $\lambda>0$)
\[
\Omega(\theta)=\frac{\lambda}{SA}\sum_{s\in\mathcal{S},a\in\mathcal{A}}\log\pi_{\theta}(a|s).
\] 
\item The initial distribution is component-wisely positive, \ie, $\rho(s)>0$ for any $s\in\mathcal{S}$. 
\item  The reward function $r(s,a)\in[0,1]$, $\forall\,s\in\mathcal{S},\,a\in\mathcal{A}$. 
\end{itemize}
\end{assumption}
\noindent Some remarks on Assumption \ref{setting}:
\begin{itemize}
\item The soft-max policy parametrization is simple yet forms the basis of the widely-used (neural network) energy based policies \cite{haarnoja2017reinforcement}. 
\item The regularization term is a simplified version of the popular (relative) entropy regularization terms \cite{REPS, schulman2017equivalence}, and has been demonstrated to be necessary to avoid exponential lower bounds when working with the soft-max policy parametrization in \cite{li2021softmax}.
\item The positivity assumption on the initial distribution is standard in the global convergence literature of policy gradient methods \cite{agarwal2019theory,bhandari2019global,softmax_PG}. 
\item The last assumption on the range of $r$ is merely for the simplicity of the subsequent discussions and can be easily relaxed to the general constant bound $r(s,a)\in[-R_{\max},R_{\max}]$, $\forall\,s\in\mathcal{S},\,a\in\mathcal{A}$. 
\end{itemize}


\paragraph{Properties of policy gradients.}
We are now ready to provide some useful properties regarding the gradients of the discounted  and the average reward problems.


We first slightly tighten the gradient domination property established in \cite[Theorem 5.2]{agarwal2019theory} for the discounted problems by utilizing the uniform ergodic property in Assumption \ref{finite-ergodic}.
\begin{proposition}\label{kakade_prop}
\emph{\textbf{(Gradient domination for discounted problems)}}
Given Assumptions \ref{finite-ergodic} and \ref{setting}, suppose that $\|\nabla_{\theta}
L^\gamma(\theta)\|_2\leq  \lambda/(2SA)$.
Then
\[
V^{\gamma,\star}-V^{\gamma}(\pi_{\theta})\leq \lambda\min\left\{
\left\|\frac{d_{\rho}^{\gamma,\pi^{\gamma,\star}}}{\rho}\right\|_{\infty},\frac{S\|d_{\rho}^{\gamma,\pi^{\gamma,\star}}\|_{\infty}}{1-\alpha_{p,S,A}}\right\}.
\] 
\end{proposition}
Here for any (randomized) policy $\pi:\mathcal{S}\rightarrow\mathcal{P}(\mathcal{A})$,
\[
d_{\rho}^{\gamma,\pi}(s)=(1-\gamma)\sum_{t=0}^{\infty}\gamma^t\prob_\rho^{\pi}(s_t
=s)
\] is the discounted state visitation distribution, where
$\prob_\rho^{\pi}(s_t=s)$ is the probability of arriving at $s$ in step $t$
starting from $s_0\sim\rho$ following policy $\pi$ in $\mathcal{M}$. In addition,
 the division in $d_{\rho}^{\gamma,\pi^\star}/\rho$ is component-wise. 
 
 
 We next establish analogously the gradient domination property for the average reward problem.
\begin{lemma}\label{grad_dom_average}
\emph{\textbf{(Gradient domination for average reward problems)}}
Given Assumptions \ref{finite-ergodic} and \ref{setting}, suppose that $\|\nabla_{\theta}\bar{L}(\theta)\|_2\leq \lambda/(2SA)$. Then 
\[
\eta^\star-\eta(\pi_{\theta})\leq  \lambda \frac{S\|\mu_{\pi^\star}\|_\infty}{1-\alpha_{p,S,A}},
\]
where $\mu_{\pi^\star}$ and $\alpha_{p,S,A}$ are defined as in Proposition \ref{dobrushin_uniform}.
\end{lemma}
The two statements above on gradient domination capture the sub-optimality results for policies satisfying certain gradient conditions.

Now recall the strongly smoothness property of the objectives for discounted problems \cite{agarwal2019theory}.
\begin{proposition}\label{Lsmooth}
\emph{\textbf{(Strongly smoothness for discounted problems
 \mbox{\cite[Lemma D.4]{agarwal2019theory}})}}
Given Assumptions \ref{finite-ergodic} and \ref{setting},
$L^\gamma$ is strongly smooth with parameter $\beta_{\lambda}=
\frac{8}{(1-\gamma)^3}+\frac{2\lambda}{S}$, \ie, 
\[
\|\nabla_{\theta}
L^\gamma(\theta_1)-\nabla_{\theta}L^\gamma(\theta_2)\|_2\leq
\beta_\lambda\|\theta_1-\theta_2\|_2
\] 
for any $\theta_1,\theta_2\in\Theta$.
\end{proposition}

We can establish analogously the strongly smoothness property for the average reward problem.
\begin{lemma}\label{smooth_pg_ave}
\emph{\textbf{(Strongly smoothness for average reward problems)}}
Under Assumptions \ref{finite-ergodic} and \ref{setting},
$\bar{L}$ is strongly smooth with parameter $\bar{\beta}_{\lambda}=22\sqrt{S}\left(\frac{2C_{p,S,A}}{1-\alpha_{p,S,A}}+1\right)^3+2\lambda /S$, \ie, 
\[
\|\nabla_{\theta}\bar{L}(\theta_1)-\nabla_{\theta}\bar{L}(\theta_2)\|_2\leq \bar{\beta}_{\lambda}\|\theta_1-\theta_2\|_2,
\] 
for any $\theta_1,\theta_2\in\Theta$. Here the constants $C_{p,S,A}>1$ and $\alpha_{p,S,A}\in[0,1)$ are defined as in Proposition \ref{dobrushin_uniform}.
\end{lemma}

These two statements are critical for the subsequent analyses of the algorithms. 


\section{DAE REINFORCE algorithm}\label{dae_reinforce}
In this section, we first introduce a widely used vanilla policy gradient implementation \cite{pg_impl},  which we call the DAE REINFORCE algorithm (following its usage of DAE in \cite{schulman2015high}). In DAE REINFORCE, a stationary parametrized policy $\pi_{\theta}(a|s)$ is considered, and the parameter is updated by 
\begin{equation}\label{vpg_spinningup1}
\theta^{k+1}=\theta^k+\alpha^k\hat{g}_k,
\end{equation}
where 
\[
\begin{split}
\hat{g}_k=&\,\dfrac{1}{NH}\sum\nolimits_{i=1}^N\sum\nolimits_{h=0}^{H-1}\nabla_{\theta}\log\pi_{\theta^k}(a_h^i|s_h^i)\overbrace{\left(\sum\nolimits_{h'=h}^{H-1}\gamma^{h'-h}r_{h'}^i-b(s_h^i)\right)}^{\text{advantage function}}+\nabla_\theta \Omega(\theta^k).
\end{split}
\]
Here $\gamma\in(0,1)$ is a fictitious discount factor, $N$ is the mini-batch size of the updates, $r_h^i=r(s_h^i,a_h^i)$, $\tau_i=(s_0^i,a_0^i,r_0^i,\dots,s_{H-1}^i,a_{H-1}^i,r_{H-1}^i)$ ($i=1,\dots,N$, $h=0,\dots,H-1$) are i.i.d. trajectories sampled under policy $\pi_{\theta^k}$, and $b$ is a baseline function that is independent of the trajectories. Throughout the paper, we assume that the baseline $b$ is a.s. uniformly bounded, \ie, $\max_{s\in\mathcal{S}}|b(s)|\leq  B$ a.s. for some constant $B>0$.

In the rest of the section, we establish the convergence of (a slightly modified version of) DAE REINFORCE, which we call Truncated DAE REINFORCE and summarize in Algorithm \ref{spinup-reinforce}. Note that the estimator $\hat{g}_k$ is truncated in \eqref{vpg_spinningup_beta} (and for notational simplicity under the same symbol) with a truncation parameter $\beta\in(0,1)$. The same truncation has been adopted for studying the standard REINFORCE algorithm (without DAE) in \cite{zhang2020sample}, where $\beta$ is introduced to ensure that the advantage function estimation is sufficiently accurate.\footnote{In \S\ref{doubly_discounted}, we show that $\beta$ can be dropped if an additional discount factor is introduced in the gradient estimator.}  
\begin{algorithm}[ht]
\caption{Truncated DAE REINFORCE}
\label{spinup-reinforce}
\begin{algorithmic}[1]
\STATE \textbf{Input:} Initialization $\theta^0$, step-sizes $\alpha^k$ for $k\geq 0$.
\FOR{$k=0,1,\dots$}
\STATE Sample $N$ i.i.d. trajectories $\{\tau_i\}_{i=1}^N$ under policy $\pi_{\theta^k}$.
\STATE Compute gradient estimator $\hat{g}_k$ as 
\begin{equation}\label{vpg_spinningup_beta}
\begin{split}
\hat{g}_k&=\dfrac{1}{N\lfloor\beta H\rfloor}\sum_{i=1}^N\sum_{h=0}^{\lfloor \beta H\rfloor-1}\nabla_{\theta}\log\pi_{\theta^k}(a_h^i|s_h^i)\left(\sum\nolimits_{h'=h}^{H-1}\gamma^{h'-h}r_{h'}^i-b(s_h^i)\right)+\nabla_\theta \Omega(\theta^k).
\end{split}
\end{equation}
\STATE Update $\theta^{k+1}=\theta^k+\alpha^k \hat{g}_k$.
\ENDFOR
\end{algorithmic}
\end{algorithm}




The main idea behind our convergence analysis is to use the average reward as a bridge to connect the original finite-time-horizon MDP and the DAE REINFORCE algorithm. The proof consists of two parts. The first part is to establish the sub-optimality of $\theta_k$, evaluated for the average reward problem. 
The second part is to establish the convergence of the algorithm for the finite-horizon problem by utilizing the connection between the average reward setting and the finite-horizon setting.

We begin the analysis by estimating the (upper) bound on the difference between the exact gradient and the sample gradient. Hereafter, we use $\Expect_k$ to denote the conditional expectation given the $k$-th iteration $\theta^k$. 
\begin{lemma}\label{spinup_grad_error_lemma}
Given Assumptions \ref{finite-ergodic} and \ref{setting},  
then
\BEQ\label{spinup_grad_error}
\begin{split}
\left\|\Expect_k[\hat{g}_k]-\nabla \bar{L}(\theta^k)\right\|_2&\leq \frac{16C_{p,S,A}}{\lfloor \beta H\rfloor (1-\alpha_{p,S,A})}\left(1+\frac{C_{p,S,A}}{1-\alpha_{p,S,A}}\right)\\
&\,\quad+8C_{p,S,A}\frac{1-\gamma}{(1-\alpha_{p,S,A})^2}+4\gamma^{(1-\beta)H}\left(1+\frac{C_{p,S,A}}{1-\alpha_{p,S,A}}\right).
\end{split}
\EEQ
Here the constants $C_{p,S,A}>1$ and $\alpha_{p,S,A}\in[0,1)$ are defined in Proposition \ref{dobrushin_uniform}.
\end{lemma}

This lemma leads to the following bounds on the stochastic gradients, which are key to establishing the convergence  of Algorithm \ref{spinup-reinforce}.
\begin{lemma}\label{spinup_stoc_grad_bds}
Given Assumptions \ref{finite-ergodic} and \ref{setting}, then
\[
\begin{split}
\|\hat{g}_k\|_2&\leq G^{\gamma}+2\lambda\quad \text{a.s.},\\
\Expect_k\hat{g}_k^T\nabla_{\theta} \bar{L}(\theta^k)&\geq \|\nabla_\theta L^\gamma(\theta^k)\|_2^2-(\bar{G}+2\lambda)\bar{\Delta},\\
\Expect_k\|\hat{g}_k\|_2^2&\leq 2\|\nabla_{\theta}\bar{L}(\theta^k)\|_2^2+\bar{M}.
\end{split}
\]
Here $G^\gamma=\frac{2(1+(1-\gamma)B)}{1-\gamma}$, $\bar{G}=4\left(1+\frac{C_{p,S,A}}{1-\alpha_{p,S,A}}\right)$, $\bar{M}=2\bar{\Delta}^2 + (G^{\gamma}+2\lambda)^2/ N$, $\bar{\Delta}$ is the right-hand side of \eqref{spinup_grad_error},  the constants $C_{p,S,A}>1$ and $\alpha_{p,S,A}\in[0,1)$ are defined in Proposition \ref{dobrushin_uniform}.
\end{lemma}
\begin{remark}
The second bound in Lemma \ref{spinup_stoc_grad_bds} shows that $\hat{g}_k$ is nearly unbiased, while the third bound shows that $\hat{g}_k$ satisfies a bounded second-order moment growth condition. These conditions slightly generalize the standard ones used  in analyzing stochastic gradient methods \cite{sgd_bottou}. 
\end{remark}


Now, we obtain first the sub-optimality behavior of $\theta^k$ from the Truncated DAE REINFORCE algorithm (\cf Algorithm \ref{spinup-reinforce}) in the average reward setting. 
\begin{theorem}\label{sample_complexity_eta_thm}
Given Assumptions \ref{finite-ergodic} and \ref{setting}, let $\bar{\beta}_{\lambda}=22\sqrt{S}\left(\frac{2C_{p,S,A}}{1-\alpha_{p,S,A}}+1\right)^3+2\lambda /S$. 
For a fixed $\beta\in(0,1)$ and any $\epsilon>0$, $\delta\in(0,1)$,  set $\alpha^k=\frac{1}{2\bar{\beta}_{\lambda}}\frac{1}{\sqrt{k+3}\log_2(k+3)}$ and $\lambda$ is the positive (larger) root of the following quadratic equation:
\[
2(\bar{G}+2\lambda)\bar{\Delta}=(\lambda-\epsilon)^2/(4S^2A^2),
\]
where $\bar{G}$ and $\bar{\Delta}$ are defined as in Lemma  \ref{spinup_stoc_grad_bds}.  
Then  
\BEQ\label{sample_complexity_eta}
\begin{split}
\min_{k=0,\dots,K}\eta^\star-\eta(\pi_{\theta^k})\leq& \frac{\|\mu_{\pi^\star}\|_\infty}{1-\alpha_{p,S,A}}
\left(S\epsilon+8S^3A^2\bar{\Delta}+4S^2A\sqrt{\bar{\Delta}\epsilon+4S^2A^2\bar{\Delta}^2+\bar{G}\bar{\Delta}}\right)
\end{split}
\EEQ
with probability at least $1-\delta$, for any $K$ such that 
\BEQ\label{K_lower_bound_eta}
\begin{split}
K\geq &\,O\left(\dfrac{S^4A^4\bar{\beta}_{\lambda}^2(\bar{D}+\sqrt{2\bar{C}\log(2/\delta)})^2}{\epsilon^4}\log^2 \left(\dfrac{SA\bar{\beta}_{\lambda}(\bar{D}+\sqrt{2\bar{C}\log(2/\delta)})}{\epsilon}\right)\right).
\end{split}
\EEQ
Here the constants $C_{p,S,A}>1$ and $\alpha_{p,S,A}\in[0,1)$ are defined in Proposition \ref{dobrushin_uniform}, and the constants $\bar{D}$ and $\bar{C}$ are bounded by 
\BEQ\label{D_bar_C_bar}
\begin{split}
&\bar{D}=\,O(\bar{M}+\lambda+1),\\
&\bar{C}=\,O\left(\frac{(G^{\gamma}+2\lambda)^2}{S}\left(\frac{C_{p,S,A}^2}{(1-\alpha_{p,S,A})^2}+\lambda^2+(G^{\gamma}+2\lambda)^2\right)\right),
\end{split}
\EEQ
where the constants hidden in the big-$O$ notation may depend on $\theta^0$. 
\end{theorem}

Next, by Lemma \ref{VH-eta-pi} and  Lemma \ref{VH-eta}, we have the following theorem.

\begin{theorem}\label{main_res_finite_average}
Given Assumption \ref{finite-ergodic}, for any $H\geq 1$, if there exists a policy $\hat{\pi}$ such that $|\eta^\star-\eta(\hat{\pi})|\leq \epsilon$ for some $\epsilon>0$, then
\BEQ\label{main_res_finite_average_bd}
\begin{split}
V^{H,\star}-V^H(\hat{\pi})
&\leq \frac{2R_{\max}D_{p,S,A}}{H}+\epsilon+\frac{2R_{\max}C_{p,S,A}}{H(1-\alpha_{p,S,A})}.
\end{split}
\EEQ
Here the constants $C_{p,S,A}> 1$, $D_{p,S,A}> 1$ and $\alpha_{p,S,A}\in[0,1)$ are the constants in Proposition \ref{dobrushin_uniform} and Lemma \ref{VH-eta}, which 
depend only on the transition probability model $p$, the number of states $S$ and the number of actions $A$ of  the underlying MDP $\mathcal{M}$. 
\end{theorem}

 Combining Theorems \ref{sample_complexity_eta_thm} and \ref{main_res_finite_average} we can derive the convergence for Truncated DAE REINFORCE algorithm.  
\begin{theorem}\label{main_conv_spin-up}
Given Assumptions \ref{finite-ergodic} and \ref{setting}, let $\gamma=1-H^{-\sigma}$ for some $\sigma\in(0,1)$. For a fixed $\beta\in(0,1)$ and any $\epsilon>0$, $\delta\in(0,1)$, set $\lambda$, $\bar{\beta}_\lambda$ and $\alpha^k$ to be the same as in Theorem \ref{sample_complexity_eta_thm}. Then for any $K$ such that \eqref{K_lower_bound_eta} is satisfied, \footnote{See Appendix \ref{formal_statement_dae_reinforce} for more explicit bounds on the constants involved in \eqref{K_lower_bound_eta}.} 
with probability at least $1-\delta$, 
\BEQ\label{sample_complexity_VH_spinup}
\begin{split}
&\min_{k=0,\dots,K}V^{H,\star}-V^{H}(\pi_{\theta^k})\leq O\left(\frac{S}{1-\alpha_{p,S,A}}\epsilon\right)+\emph{\text{\bf bias}}^{\rm DAE}_H,
\end{split}
\EEQ
where
\[
\begin{split}
\emph{\text{\bf bias}}^{\rm DAE}_H=&\,O\left(\frac{S^2AC_{p,S,A}^3}{(1-\alpha_{p,S,A})^4}H^{-\sigma/2}+\frac{S^3A^2C_{p,S,A}^2}{(1-\alpha_{p,S,A})^3}H^{-\sigma}+\left(D_{p,S,A}+\frac{C_{p,S,A}}{1-\alpha_{p,S,A}}\right)H^{-1}\right).
\end{split}
\]
Here $C_{p,S,A}> 1$, $D_{p,S,A}> 1$ and $\alpha_{p,S,A}\in[0,1)$ are constants in Proposition \ref{dobrushin_uniform} and Lemma \ref{VH-eta}.
\end{theorem}
The choice of $\gamma$ is for ease of presentation. See also \cite{liu2020gamma, dong2021simple}.

 \section{Doubly Discounted REINFORCE algorithm}\label{doubly_discounted}
In Algorithm \ref{spinup-reinforce}, a fictitious discount factor is introduced when computing advantage function estimates, while for the rest part it remains undiscounted. This introduces a bias term $\emph{\text{\bf bias}}^{\rm DAE}_H$ as shown in Theorem \ref{main_conv_spin-up}, which remains nonzero for a fixed planning horizon $H$ even when the number of iterations $K$ goes to infinity and $\epsilon$ goes to $0$.
In this section, we propose  the Doubly Discounted REINFORCE algorithm (\cf Algorithm \ref{dd-reinforce}) to reduce the bias introduced by DAE.  
\begin{algorithm}[ht]
\caption{Doubly Discounted REINFORCE}
\label{dd-reinforce}
\begin{algorithmic}[1]
\STATE \textbf{Input:} Initialization $\theta^0$, step-sizes $\alpha^k$ for $k\geq 0$.
\FOR{$k=0,1,\dots$}
\STATE Sample $N$ i.i.d. trajectories $\{\tau_i\}_{i=1}^N$ under policy $\pi_{\theta^k}$.
\STATE Compute gradient estimator $\tilde{g}_k$ as 
\begin{equation}\label{vpg_discount}
\begin{split}
&\tilde{g}_k=\dfrac{1}{N}\sum\limits_{i=1}^N\sum\limits_{h=0}^{H-1}\gamma^h\nabla_{\theta}\log\pi_{\theta^k}(a_h^i|s_h^i)\left(\sum\limits_{h'=h}^{H-1}\gamma^{h'-h}r_{h'}^i-b(s_h^i)\right)+\nabla_\theta \Omega(\theta^k).
\end{split}
\end{equation}
\STATE Update $\theta^{k+1}=\theta^k+\alpha^k \tilde{g}_k$.
\ENDFOR
\end{algorithmic}
\end{algorithm}

Compared with Algorithm \ref{spinup-reinforce}, Algorithm \ref{dd-reinforce} introduces an additional discount factor when computing the score functions and gets rid of the artificial parameter $\beta\in(0,1)$ needed in Truncated DAE REINFORCE. As a result, the estimator \eqref{vpg_discount} coincides with the vanilla policy gradient estimator for solving discounted problems \cite{zhang2020sample} with a fixed-length trajectory truncation \cite{liu2020improved}. Note that a similar observation has been made for natural actor-critic methods in \cite{thomas2014bias}. 



Similar to the idea of \S\ref{dae_reinforce}, we first establish the sub-optimality of the Doubly Discounted REINFORCE algorithm, evaluated for the discounted problem. Parallel to Lemma \ref{spinup_stoc_grad_bds}, we have the following stochastic gradient bounds.

\begin{lemma}\label{stoc_grad_bds}
Given Assumptions \ref{finite-ergodic} and \ref{setting}, then 
\[
\begin{split}
\|\tilde{g}_k\|_2&\leq G+2\lambda\quad \text{a.s.},\\
\Expect_k\tilde{g}_k^T\nabla_{\theta} L^\gamma(\theta^k)&\geq \|\nabla_\theta L^\gamma(\theta^k)\|_2^2-(G+2\lambda)\Delta,\\
\Expect_k\|\tilde{g}_k\|_2^2&\leq 2\|\nabla_{\theta}L^{\gamma}(\theta^k)\|_2^2+M.
\end{split}
\]
Here $G=\frac{2(1+B(1-\gamma))}{(1-\gamma)^2}$,  
and the constants $\Delta$ and $M$ are defined by
\[
\Delta=2\dfrac{\gamma^H}{1-\gamma}\left(H+\frac{1}{1-\gamma}\right),\quad M=2\Delta^2+(G+2\lambda)^2/N.
\] 
\end{lemma}

Based on the above conditions, we now establish the sub-optimality of $\theta^k$ from the Doubly Discounted REINFORCE algorithm for the discounted problem.
\begin{theorem}\label{sample_complexity_Vgamma_thm}
Given Assumptions \ref{finite-ergodic} and \ref{setting}, let $\beta_{\lambda}=8/(1-\gamma)^3+2\lambda/S$. For any $\epsilon>0$ and $\delta\in(0,1)$, set $\alpha^k=\frac{1}{2\beta_{\lambda}}\frac{1}{\sqrt{k+3}\log_2(k+3)}$ and $\lambda$ to be the positive (larger) root of the following quadratic equation:
\[
2(G+2\lambda)\Delta=(\lambda-\epsilon)^2/(4S^2A^2).
\]
Then
\BEQ\label{sample_complexity_Vgamma}
\begin{split}
\min_{k=0,\dots,K}V^{\gamma,\star}-V^{\gamma}(\pi_{\theta^k})&\leq \,\min\left\{\left\|\frac{d_{\rho}^{\gamma,\pi^{\gamma,\star}}}{\rho}\right\|_{\infty},\frac{S\|d_{\rho}^{\gamma,\pi^{\gamma,\star}}\|_{\infty}}{1-\alpha_{p,S,A}}\right\}\\
&\,\quad\times (\epsilon+8S^2A^2\Delta+4SA\sqrt{\Delta\epsilon+4S^2A^2\Delta^2+G\Delta})
\end{split}
\EEQ
with probability at least $1-\delta$, for any $K$ such that 
\BEQ\label{K_lower_bound}
\begin{split}
K\geq &\,O\left(\dfrac{S^4A^4\beta_{\lambda}^2(D+\sqrt{2C\log(2/\delta)})^2}{\epsilon^4}\log^2 \left(\dfrac{SA\beta_{\lambda}(D+\sqrt{2C\log(2/\delta)})}{\epsilon}\right)\right).
\end{split}
\EEQ
Here the constant $\alpha_{p,S,A}\in(0,1)$ is defined in Proposition \ref{dobrushin_uniform}, and the constants $D$ and $C$ are bounded by 
\BEQ\label{D_C}
\begin{split}
&D=O(M+1/(1-\gamma)+\lambda),\\
&C=O((G+2\lambda)^2(1/(1-\gamma)^4+\lambda^2+(G+2\lambda)^2)),
\end{split}
\EEQ
where the constants hidden in the big-$O$ notation may depend on $\theta^0$. 
\end{theorem}

The next result is parallel to Theorem \ref{main_res_finite_average}, and is based on Lemma \ref{error_bd_arb}, Corollary \ref{Vgamma-eta}, and Lemma \ref{VH-eta}.

\begin{theorem}\label{main_res_finite_discount}
Given Assumption \ref{finite-ergodic}, if there exists a policy $\hat{\pi}$ such that $V^{\gamma,\star}-V^{\gamma}(\hat{\pi})\leq \epsilon$ for some $\epsilon>0$, then for any $H\geq 1$, 
\BEQ\label{main_res_finite_discount_bd}
\begin{split}
V^{H,\star}-V^H(\hat{\pi})\leq &\,\,2R_{\max}C_{p,S,A}\frac{\gamma}{H(1-\gamma)}\alpha_{p,S,A}^H+\epsilon\\
&\,+\dfrac{2R_{\max}}{H}\left(\dfrac{C_{p,S,A}(H(1-\gamma)+\alpha_{p,S,A}+|H(1-\gamma)-1|)}{1-\alpha_{p,S,A}}+D_{p,S,A}\right),
\end{split}
\EEQ
where  $C_{p,S,A}> 1$, $D_{p,S,A}> 1$ and $\alpha_{p,S,A}\in[0,1)$ are the constants in Proposition \ref{dobrushin_uniform} and Lemma \ref{VH-eta}, which 
depend only on the transition probability model $p$, the number of states $S$ and the number of actions $A$ of  the underlying MDP $\mathcal{M}$. 
\end{theorem}

Combining Theorems \ref{sample_complexity_Vgamma_thm} and  \ref{main_res_finite_discount}, we obtain the final convergence result for the Doubly Discounted REINFORCE algorithm (in parallel to Theorem \ref{main_conv_spin-up}). 

\begin{theorem}\label{discounted_reinforce_thm}
Given Assumptions \ref{finite-ergodic} and \ref{setting}, let $\gamma=1-H^{-\sigma}$ for some $\sigma\in(0,1)$. For any $\epsilon>0$, $\delta\in(0,1)$, set $\lambda$, $\beta_\lambda$ and $\alpha^k$ to be the same as in Theorem \ref{sample_complexity_Vgamma_thm}. Then for any $K$ such that \eqref{K_lower_bound} is satisfied, 
 \footnote{See Appendix \ref{formal_statement_dd_reinforce} for more explicit bounds on the constants involved in \eqref{K_lower_bound}.} 
with probability at least $1-\delta$, 
\BEQ\label{sample_complexity_VH}
\begin{split}
&\min_{k=0,\dots,K}V^{H,\star}-V^{H}(\pi_{\theta^k}) \leq O\left(\epsilon\min\left\{\left\|\frac{1}{\rho}\right\|_\infty, \frac{S}{1-\alpha_{p,S,A}}\right\}\right) + \emph{\text{\bf bias}}^{\rm DD}_H,
\end{split}
\EEQ
where
\[
\begin{split}
 \emph{\text{\bf bias}}^{\rm DD}_H&=\,O\left(\frac{C_{p,S,A}}{1-\alpha_{p,S,A}}H^{-\sigma}+D_{p,S,A}H^{-1}+\frac{S^3A^2}{1-\alpha_{p,S,A}}H^{\frac{1+3\sigma}{2}}e^{-H^{1-\sigma}/2}+C_{p,S,A}\alpha_{p,S,A}^HH^{-(1-\sigma)}\right).
\end{split}
\]
Here $C_{p,S,A}> 1$, $D_{p,S,A}> 1$ and $\alpha_{p,S,A}\in[0,1)$ are constants in Proposition \ref{dobrushin_uniform} and Lemma \ref{VH-eta}.
\end{theorem}

\paragraph{Comparison with  DAE REINFORCE.}
Here we compare the convergence of (truncated) DAE REINFORCE (\cf Algorithm \ref{spinup-reinforce}) and Doubly Discounted REINFORCE (\cf Algorithm \ref{dd-reinforce}). Note that in both \eqref{sample_complexity_VH_spinup} and \eqref{sample_complexity_VH}, the global sub-optimality bounds consist of two parts: a vanishing $\epsilon$ term that goes to zero as the number of iterations $K$ goes to infinity and a remaining bias term ($\emph{\text{\bf bias}}^{\rm DAE}_H$ and $\emph{\text{\bf bias}}^{\rm DD}_H$, respectively) resulting from the fictitious discount factor. 
Below we focus on comparing the bias terms 
with the same fictitious discount factor $\gamma=1-H^{-\sigma}$, with $\sigma\in(0,1)$.   Recall that  
\[
\begin{split}
&\emph{\text{\bf bias}}^{\rm DAE}_H=O\left(\frac{S^2AC_{p,S,A}^3}{(1-\alpha_{p,S,A})^4}H^{-\frac{\sigma}{2}}\right) + \emph{\text{lower order terms in $H$}},\\
&\emph{\text{\bf bias}}^{\rm DD}_H=\,O\left(\frac{C_{p,S,A}}{1-\alpha_{p,S,A}}H^{-\sigma} \right) + \emph{\text{lower order terms in $H$}}.
\end{split}
\]
Comparing the above two bounds, 
we see the power of the additional discounting. Indeed, with further discounting, Doubly Discounted REINFORCE improves over DAE REINFORCE, especially in terms of $H$ (from $H^{-\sigma/2}$ to $H^{-\sigma}$) as it grows. More precisely,  the constant before the $H^{-\sigma}$ term is improved from $O(S^3A^2C_{p,S,A}^2/(1-\alpha_{p,S,A})^3)$ to $O(C_{p,S,A}/(1-\alpha_{p,S,A}))$, the constant before the $H^{-1}$ term is improved from $O(D_{p,S,A}+C_{p,S,A}/(1-\alpha_{p,S,A})$ to $O(D_{p,S,A})$, while the $H^{-\sigma/2}$ term is improved to be exponentially decaying as $H$ grows. 

\section{Conclusion and extensions}\label{extension}
This paper  focuses on two concrete fictitious discount algorithms in the context of policy gradient methods, namely DAE REINFORCE and Doubly Discounted REINFORCE. Rigorous convergence analyses are established for the two algorithms, which, for the first time, shed light on the non-asymptotic global convergence of fictitious discount algorithms.

 Given recent development in (global) convergence analysis of algorithms in the discounted setting  \cite{agarwal2019theory,wang2019neural,shani2019adaptive} and in the average reward framework \cite{neu2017unified,POLITEX},  
it is natural to extend our study for natural policy gradient \cite{kakade2001natural}, natural actor-critic \cite{peters2008natural}, TRPO \cite{TRPO}, PPO \cite{PPO}, as well as deep learning based algorithms such as DQN \cite{DQN_original} and DDPG \cite{DDPG}. 

Meanwhile, it remains to see if one can generalize our work to the general weakly communicating MDPs \cite{regal} or MDPs with more general state and action spaces, and to remove the need for an exploratory initial distribution (\ie, $\rho>0$ component-wisely) (\eg, by combining with the policy cover approach in \cite{PC-PG}).  

\bibliographystyle{plain}
\bibliography{fderl}

\begin{thebibliography}{10}

\bibitem{POLITEX}
Y.~Abbasi-Yadkori, P.~Bartlett, K.~Bhatia, N.~Lazic, C.~Szepesvari, and
  G.~Weisz.
\newblock Politex: Regret bounds for policy iteration using expert prediction.
\newblock In {\em International Conference on Machine Learning}, pages
  3692--3702, 2019.

\bibitem{pg_impl}
J.~Achiam.
\newblock {OpenAI Spinning Up}: Vanilla policy gradient, 2018.

\bibitem{PC-PG}
A.~Agarwal, M.~Henaff, S.~Kakade, and W.~Sun.
\newblock {PC-PG}: Policy cover directed exploration for provable policy
  gradient learning.
\newblock {\em arXiv preprint arXiv:2007.08459}, 2020.

\bibitem{agarwal2019reinforcement}
A.~Agarwal, N.~Jiang, and S.~Kakade.
\newblock Reinforcement {L}earning: {T}heory and {A}lgorithms.
\newblock Technical report, Department of Computer Science, University of
  Washington, 2019.

\bibitem{agarwal2019theory}
Alekh Agarwal, Sham~M. Kakade, Jason~D Lee, and Gaurav Mahajan.
\newblock On the theory of policy gradient methods: Optimality, approximation,
  and distribution shift.
\newblock {\em arXiv preprint arXiv:1908.00261}, 2019.

\bibitem{amit2020discount}
Ron Amit, Ron Meir, and Kamil Ciosek.
\newblock Discount factor as a regularizer in reinforcement learning.
\newblock In {\em International conference on machine learning}, pages
  269--278. PMLR, 2020.

\bibitem{regal}
P.~Bartlett and A.~Tewari.
\newblock {REGAL}: A regularization based algorithm for reinforcement learning
  in weakly communicating mdps.
\newblock {\em arXiv preprint arXiv:1205.2661}, 2012.

\bibitem{baxter2001infinite}
J.~Baxter and P.~Bartlett.
\newblock Infinite-horizon policy-gradient estimation.
\newblock {\em Journal of Artificial Intelligence Research}, 15:319--350, 2001.

\bibitem{baxter1999direct}
Jonathan Baxter and Peter~L. Bartlett.
\newblock Direct gradient-based reinforcement learning: I. gradient estimation
  algorithms.
\newblock Technical report, Citeseer, 1999.

\bibitem{bhandari2019global}
J.~Bhandari and D.~Russo.
\newblock Global optimality guarantees for policy gradient methods.
\newblock {\em arXiv preprint arXiv:1906.01786}, 2019.

\bibitem{blackwell1962discrete}
David Blackwell.
\newblock Discrete dynamic programming.
\newblock {\em The Annals of Mathematical Statistics}, pages 719--726, 1962.

\bibitem{sgd_bottou}
L.~Bottou, F.~Curtis, and J.~Nocedal.
\newblock Optimization methods for large-scale machine learning.
\newblock {\em SIAM Review}, 60(2):223--311, 2018.

\bibitem{RND}
Yuri Burda, Harrison Edwards, Amos Storkey, and Oleg Klimov.
\newblock Exploration by random network distillation.
\newblock {\em arXiv preprint arXiv:1810.12894}, 2018.

\bibitem{cen2020fast}
S.~Cen, C.~Cheng, Y.~Chen, Y.~Wei, and Y.~Chi.
\newblock Fast global convergence of natural policy gradient methods with
  entropy regularization.
\newblock {\em arXiv preprint arXiv:2007.06558}, 2020.

\bibitem{dong2021simple}
Shi Dong, Benjamin Van~Roy, and Zhengyuan Zhou.
\newblock Simple agent, complex environment: Efficient reinforcement learning
  with agent state.
\newblock {\em arXiv preprint arXiv:2102.05261}, 2021.

\bibitem{even2009online}
Eyal Even-Dar, Sham~M. Kakade, and Yishay Mansour.
\newblock Online markov decision processes.
\newblock {\em Mathematics of Operations Research}, 34(3):726--736, 2009.

\bibitem{fedus2019hyperbolic}
William Fedus, Carles Gelada, Yoshua Bengio, Marc~G Bellemare, and Hugo
  Larochelle.
\newblock Hyperbolic discounting and learning over multiple horizons.
\newblock {\em arXiv preprint arXiv:1902.06865}, 2019.

\bibitem{feinberg1996constrained}
Eugene~A Feinberg and Adam Shwartz.
\newblock Constrained discounted dynamic programming.
\newblock {\em Mathematics of Operations Research}, 21(4):922--945, 1996.

\bibitem{franccois2015discount}
Vincent Fran{\c{c}}ois-Lavet, Raphael Fonteneau, and Damien Ernst.
\newblock How to discount deep reinforcement learning: Towards new dynamic
  strategies.
\newblock {\em arXiv preprint arXiv:1512.02011}, 2015.

\bibitem{gao2017properties}
Bolin Gao and Lacra Pavel.
\newblock On the properties of the softmax function with application in game
  theory and reinforcement learning.
\newblock {\em arXiv preprint arXiv:1704.00805}, 2017.

\bibitem{gergely2010online}
Andr{\'a}s~Gy{\"o}rgy Gergely~Neu, Csaba Szepesv{\'a}ri, and Andr{\'a}s Antos.
\newblock Online markov decision processes under bandit feedback.
\newblock In {\em Proceedings of the Twenty-Fourth Annual Conference on Neural
  Information Processing Systems}, 2010.

\bibitem{haarnoja2017reinforcement}
Tuomas Haarnoja, Haoran Tang, Pieter Abbeel, and Sergey Levine.
\newblock Reinforcement learning with deep energy-based policies.
\newblock In {\em International Conference on Machine Learning}, pages
  1352--1361. PMLR, 2017.

\bibitem{hambly2020policy}
Ben Hambly, Renyuan Xu, and Huining Yang.
\newblock Policy gradient methods for the noisy linear quadratic regulator over
  a finite horizon.
\newblock {\em Available at SSRN}, 2020.

\bibitem{hambly2021policy}
Ben Hambly, Renyuan Xu, and Huining Yang.
\newblock Policy gradient methods find the {N}ash equilibrium in n-player
  general-sum linear-quadratic games.
\newblock {\em Available at SSRN 3894471}, 2021.

\bibitem{haviv1984perturbation}
Moshe Haviv and Ludo Van~der Heyden.
\newblock Perturbation bounds for the stationary probabilities of a finite
  markov chain.
\newblock {\em Advances in Applied Probability}, pages 804--818, 1984.

\bibitem{henkel2018regularization}
Florian Henkel.
\newblock {\em A Regularization Study for Policy Gradient Methods/submitted by
  Florian Henkel}.
\newblock PhD thesis, Universit{\"a}t Linz, 2018.

\bibitem{hessel2018rainbow}
Matteo Hessel, Joseph Modayil, Hado Van~Hasselt, Tom Schaul, Georg Ostrovski,
  Will Dabney, Dan Horgan, Bilal Piot, Mohammad Azar, and David Silver.
\newblock Rainbow: Combining improvements in deep reinforcement learning.
\newblock In {\em Thirty-second AAAI conference on artificial intelligence},
  2018.

\bibitem{hordijk2002blackwell}
Arie Hordijk and Alexander~A. Yushkevich.
\newblock Blackwell optimality.
\newblock In {\em Handbook of Markov decision processes}, pages 231--267.
  Springer, 2002.

\bibitem{kakade2001optimizing}
Sham Kakade.
\newblock Optimizing average reward using discounted rewards.
\newblock In {\em International Conference on Computational Learning Theory},
  pages 605--615. Springer, 2001.

\bibitem{kakade2001natural}
Sham~M. Kakade.
\newblock A natural policy gradient.
\newblock {\em Advances in neural information processing systems}, 14, 2001.

\bibitem{actor-critic}
V.~Konda and J.~Tsitsiklis.
\newblock On actor-critic algorithms.
\newblock {\em SIAM journal on Control and Optimization}, 42(4):1143--1166,
  2003.

\bibitem{lasserre1988conditions}
JB~Lasserre.
\newblock Conditions for existence of average and blackwell optimal stationary
  policies in denumerable markov decision processes.
\newblock {\em Journal of mathematical analysis and applications},
  136(2):479--489, 1988.

\bibitem{lewis2002bias}
Mark~E. Lewis and Martin~L. Puterman.
\newblock Bias optimality.
\newblock In {\em Handbook of Markov decision processes}, pages 89--111.
  Springer, 2002.

\bibitem{li2021softmax}
Gen Li, Yuting Wei, Yuejie Chi, Yuantao Gu, and Yuxin Chen.
\newblock Softmax policy gradient methods can take exponential time to
  converge.
\newblock {\em arXiv preprint arXiv:2102.11270}, 2021.

\bibitem{DDPG}
Timothy~P. Lillicrap, Jonathan~J. Hunt, Alexander Pritzel, Nicolas Heess, Tom
  Erez, Yuval Tassa, David Silver, and Daan Wierstra.
\newblock Continuous control with deep reinforcement learning.
\newblock {\em arXiv preprint arXiv:1509.02971}, 2015.

\bibitem{liu2020gamma}
Shuang Liu and Hao Su.
\newblock $\gamma$-regret for non-episodic reinforcement learning.
\newblock {\em arXiv e-prints}, pages arXiv--2002, 2020.

\bibitem{liu2020improved}
Yanli Liu, Kaiqing Zhang, Tamer Basar, and Wotao Yin.
\newblock An improved analysis of (variance-reduced) policy gradient and
  natural policy gradient methods.
\newblock {\em Advances in Neural Information Processing Systems}, 33, 2020.

\bibitem{mahadevan1996sensitive}
Sridhar Mahadevan.
\newblock Sensitive discount optimality: Unifying discounted and average reward
  reinforcement learning.
\newblock In {\em ICML}, pages 328--336. Citeseer, 1996.

\bibitem{marbach2001simulation}
P.~Marbach and J.~Tsitsiklis.
\newblock Simulation-based optimization of {M}arkov reward processes.
\newblock {\em IEEE Transactions on Automatic Control}, 46(2):191--209, 2001.

\bibitem{marbach1998simulation}
Peter Marbach.
\newblock {\em Simulation-based optimization of Markov decision processes}.
\newblock PhD thesis, Massachusetts Institute of Technology, 1998.

\bibitem{marbach2003approximate}
Peter Marbach and John~N Tsitsiklis.
\newblock Approximate gradient methods in policy-space optimization of markov
  reward processes.
\newblock {\em Discrete Event Dynamic Systems}, 13(1):111--148, 2003.

\bibitem{softmax_PG}
J.~Mei, C.~Xiao, C.~Szepesvari, and D.~Schuurmans.
\newblock On the global convergence rates of softmax policy gradient methods.
\newblock {\em arXiv preprint arXiv:2005.06392}, 2020.

\bibitem{mnih2016asynchronous}
Volodymyr Mnih, Adria~Puigdomenech Badia, Mehdi Mirza, Alex Graves, Timothy
  Lillicrap, Tim Harley, David Silver, and Koray Kavukcuoglu.
\newblock Asynchronous methods for deep reinforcement learning.
\newblock In {\em International conference on machine learning}, pages
  1928--1937. PMLR, 2016.

\bibitem{DQN_original}
Volodymyr Mnih, Koray Kavukcuoglu, David Silver, Andrei~A Rusu, Joel Veness,
  Marc~G Bellemare, Alex Graves, Martin Riedmiller, Andreas~K Fidjeland, Georg
  Ostrovski, et~al.
\newblock Human-level control through deep reinforcement learning.
\newblock {\em nature}, 518(7540):529--533, 2015.

\bibitem{neu2017unified}
G.~Neu, A.~Jonsson, and V.~G{\'o}mez.
\newblock A unified view of entropy-regularized {M}arkov decision processes.
\newblock {\em arXiv preprint arXiv:1705.07798}, 2017.

\bibitem{ortner2020regret}
Ronald Ortner.
\newblock Regret bounds for reinforcement learning via {M}arkov chain
  concentration.
\newblock {\em Journal of Artificial Intelligence Research}, 67:115--128, 2020.

\bibitem{SVRPG}
M.~Papini, D.~Binaghi, G.~Canonaco, M.~Pirotta, and M.~Restelli.
\newblock Stochastic variance-reduced policy gradient.
\newblock {\em arXiv preprint arXiv:1806.05618}, 2018.

\bibitem{REPS}
J.~Peters, K.~M{\"u}lling, and Y.~Altun.
\newblock Relative entropy policy search.
\newblock In {\em AAAI}, volume~10, pages 1607--1612. Atlanta, 2010.

\bibitem{peters2008natural}
Jan Peters and Stefan Schaal.
\newblock Natural actor-critic.
\newblock {\em Neurocomputing}, 71(7-9):1180--1190, 2008.

\bibitem{puterman2014markov}
Martin~L. Puterman.
\newblock {\em Markov Decision Processes: Discrete Stochastic Dynamic
  Programming}.
\newblock John Wiley \& Sons, 2014.

\bibitem{rosenthal1995convergence}
Jeffrey~S. Rosenthal.
\newblock Convergence rates for {M}arkov chains.
\newblock {\em Siam Review}, 37(3):387--405, 1995.

\bibitem{ryu2016primer}
Ernest~K. Ryu and Stephen Boyd.
\newblock Primer on monotone operator methods.
\newblock {\em Appl. Comput. Math}, 15(1):3--43, 2016.

\bibitem{schneckenreither2020average}
Manuel Schneckenreither.
\newblock Average reward adjusted discounted reinforcement learning:
  Near-blackwell-optimal policies for real-world applications.
\newblock {\em arXiv preprint arXiv:2004.00857}, 2020.

\bibitem{schulman2017equivalence}
J.~Schulman, X.~Chen, and P.~Abbeel.
\newblock Equivalence between policy gradients and soft {Q}-learning.
\newblock {\em arXiv preprint arXiv:1704.06440}, 2017.

\bibitem{TRPO}
J.~Schulman, S.~Levine, P.~Abbeel, M.~Jordan, and P.~Moritz.
\newblock Trust region policy optimization.
\newblock In {\em International conference on machine learning}, pages
  1889--1897, 2015.

\bibitem{PPO}
J.~Schulman, F.~Wolski, P.~Dhariwal, A.~Radford, and O.~Klimov.
\newblock Proximal policy optimization algorithms.
\newblock {\em arXiv preprint arXiv:1707.06347}, 2017.

\bibitem{schulman2015high}
John Schulman, Philipp Moritz, Sergey Levine, Michael Jordan, and Pieter
  Abbeel.
\newblock High-dimensional continuous control using generalized advantage
  estimation.
\newblock {\em arXiv preprint arXiv:1506.02438}, 2015.

\bibitem{shani2019adaptive}
L.~Shani, Y.~Efroni, and S.~Mannor.
\newblock Adaptive trust region policy optimization: Global convergence and
  faster rates for regularized {MDP}s.
\newblock {\em arXiv preprint arXiv:1909.02769}, 2019.

\bibitem{stewart1990matrix}
G.~Stewart.
\newblock Matrix perturbation theory.
\newblock 1990.

\bibitem{sutton2000policy}
R.~Sutton, D.~McAllester, S.~Singh, and Y.~Mansour.
\newblock Policy gradient methods for reinforcement learning with function
  approximation.
\newblock In {\em Advances in neural information processing systems}, pages
  1057--1063, 2000.

\bibitem{tessler2020reward}
Chen Tessler and Shie Mannor.
\newblock Reward tweaking: Maximizing the total reward while planning for short
  horizons.
\newblock {\em arXiv preprint arXiv:2002.03327}, 2020.

\bibitem{thomas2014bias}
Philip Thomas.
\newblock Bias in natural actor-critic algorithms.
\newblock In {\em International conference on machine learning}, pages
  441--448. PMLR, 2014.

\bibitem{wang2019neural}
Lingxiao Wang, Qi~Cai, Zhuoran Yang, and Zhaoran Wang.
\newblock Neural policy gradient methods: Global optimality and rates of
  convergence.
\newblock {\em arXiv preprint arXiv:1909.01150}, 2019.

\bibitem{wang2014convergence}
Neng-Yi Wang, Liming Wu, et~al.
\newblock Convergence rate and concentration inequalities for {G}ibbs sampling
  in high dimension.
\newblock {\em Bernoulli}, 20(4):1698--1716, 2014.

\bibitem{xu2019sample}
Pan Xu, Felicia Gao, and Quanquan Gu.
\newblock Sample efficient policy gradient methods with recursive variance
  reduction.
\newblock {\em arXiv preprint arXiv:1909.08610}, 2019.

\bibitem{xu2018meta}
Zhongwen Xu, Hado van Hasselt, and David Silver.
\newblock Meta-gradient reinforcement learning.
\newblock {\em arXiv preprint arXiv:1805.09801}, 2018.

\bibitem{zhang2020variational}
Junyu Zhang, Alec Koppel, Amrit~Singh Bedi, Csaba Szepesvari, and Mengdi Wang.
\newblock Variational policy gradient method for reinforcement learning with
  general utilities.
\newblock {\em arXiv preprint arXiv:2007.02151}, 2020.

\bibitem{zhang2020sample}
Junzi Zhang, Jongho Kim, Brendan O'Donoghue, and Stephen Boyd.
\newblock Sample efficient reinforcement learning with {REINFORCE}.
\newblock {\em arXiv preprint arXiv:2010.11364}, 2020.

\bibitem{zhang2019global}
K.~Zhang, A.~Koppel, H.~Zhu, and T.~Ba{\c{s}}ar.
\newblock Global convergence of policy gradient methods to (almost) locally
  optimal policies.
\newblock {\em arXiv preprint arXiv:1906.08383}, 2019.

\bibitem{zhang2021derivative}
Kaiqing Zhang, Xiangyuan Zhang, Bin Hu, and Tamer Ba{\c{s}}ar.
\newblock Derivative-free policy optimization for risk-sensitive and robust
  control design: Implicit regularization and sample complexity.
\newblock {\em arXiv preprint arXiv:2101.01041}, 2021.

\bibitem{zhao2016regularized}
Tingting Zhao, Gang Niu, Ning Xie, Jucheng Yang, and Masashi Sugiyama.
\newblock Regularized policy gradients: direct variance reduction in policy
  gradient estimation.
\newblock In {\em Asian Conference on Machine Learning}, pages 333--348. PMLR,
  2016.

\end{thebibliography}

\newpage


\appendix
\section*{Appendix}
\addcontentsline{toc}{section}{Appendix}

\section{Preliminary facts}
In this section, we show the proofs of results in \S \ref{prelim}: Propositions \ref{finite-ergodic} and \ref{kakade_prop}, Lemmas \ref{error_bd_arb}, \ref{Vgamma-eta-pi}, \ref{VH-eta-pi}, \ref{VH-eta}, \ref{grad_dom_average} and \ref{smooth_pg_ave}. For ease of notation, we define $c=H(1-\gamma)$ so that $\gamma=1-c/H$ and $c\in(0,H]$.

\paragraph{Notation and terminology.} Here and below we use $P_{\pi}\in\reals^{S\times S}$ to denote the transition probability of the Markov chain induced by policy $\pi$, \ie, $P_{\pi}(s,s')=\sum_{a\in\mathcal{A}}p(s'|s,a)\pi(a|s)$. In general, a matrix $P\in\reals^{S\times S}$ is called a stochastic matrix if $P(s,s')\geq0$ for any $s,\,s'\in\mathcal{S}$ and $\sum_{s'\in\mathcal{S}}P(s,s')=1$ for any $s\in\mathcal{S}$. If in addition we also have $P(s,s')>0$ for any $s,\,s'\in\mathcal{S}$, then we say that $P$ is a positive stochastic matrix.  We say that a policy $\pi$ is deterministic if for any $s\in\mathcal{S}$, $\pi(a_s|s)=1$ for some $a_s\in\mathcal{A}$. Unless otherwise stated, all state distributions (\eg, $\rho$) are row vectors.

We also introduce the following notation to be used in the proof. The first three quantities are defined for the discounted setting, while the last three quantities are defined for the average reward setting. In all cases, $\tau=(s_0,a_0,r_0,\dots,s_h,a_h,r_h,\dots)$ is a trajectory sampled under policy $\pi$. 
\begin{itemize}
\item discounted value function: 
\[
V^{\gamma,\pi}(s)=(1-\gamma)\Expect\left[\sum_{h=0}^{\infty}\gamma^hr(s_h,a_h)\Big|s_0=s,\pi\right].
\] 
\item  discounted action-value function: 
\[
Q^{\gamma,\pi}(s,a)=(1-\gamma)\Expect\left[\sum_{h=0}^{\infty}\gamma^hr(s_h,a_h)\Big|s_0=s,a_0=a,\pi\right].
\]
\item discounted advantage function: $A^{\gamma,\pi}(s,a)=Q^{\gamma,\pi}(s,a)-V^{\gamma,\pi}(s)$.
\item average reward bias value function: 
\[
\bar{V}^{\pi}(s)=\lim_{H\rightarrow\infty}\Expect\left[\sum_{h=0}^{H-1}r(s_h,a_h)-\eta(\pi)\Big|s_0=s,\pi\right].
\] 
\item  average reward action-value function: 
\[
\bar{Q}^{\pi}(s,a)=\lim_{H\rightarrow\infty}\Expect\left[\sum_{h=0}^{H-1}r(s_h,a_h)-\eta(\pi)\Big|s_0=s,a_0=a,\pi\right].
\]
\item average reward advantage function: $\bar{A}^{\pi}(s,a)=\bar{Q}^{\pi}(s,a)-\bar{V}^{\pi}(s)$. 
\end{itemize}
Note that we have $V^{\gamma}(\pi)=\sum_{s\in\mathcal{S}}\rho(s)V^{\gamma,\pi}(s)$. 

\subsection{Proof of Proposition \ref{dobrushin_uniform}}
We first show two propositions which will be used in the proof of Proposition \ref{dobrushin_uniform}.

The following well-known fact about the convergence rate of finite and ergodic Markov chains (also known as Dobrushin's inequality) is central for our proofs. 
\begin{proposition}\label{dobrushin}
Let $P\in\reals^{S\times S}$ be a positive stochastic matrix. 
Then for any distribution $\rho\in\mathcal{P}(\mathcal{S})$ (viewed as a row vector of length $S$), we have for any $h\geq 0$,  
\BEQ\label{d_TV_alpha}
d_{\rm TV}(\rho P^h,\mu_P)\leq \alpha_{P}^h,
\EEQ
where $d_{\rm TV}$ is the total variation distance between two measures, $\mu_{P}$ is the (unique) stationary distribution of the transition matrix $P$, and $\alpha_P=1-S\min_{s,\,s'\in\mathcal{S}}P(s,s')\in[0,1)$. 
\end{proposition}
A proof of the above proposition can be found in standard textbooks \cite{rosenthal1995convergence, wang2014convergence}.       

Another important property is the following proposition for representing an arbitrary (randomized) policy as a convex combination of finitely many deterministic policies. 
\begin{proposition}\label{convex_comb}
Suppose that $|\mathcal{S}|=S<\infty$ and $|\mathcal{A}|=A<\infty$. Let $\pi:\mathcal{S}\rightarrow\mathcal{P}(\mathcal{A})$ be an arbitrary policy. Then there exist $n_{S,A}=S(A-1)+1$ deterministic policies $\pi_1,\dots,\pi_{n_{S,A}}$ and nonnegative constants $c_1,\dots,c_{SA}$, such that
\[
\pi(a|s)=\sum\nolimits_{i=1}^{n_{S,A}}c_i\pi_i(a|s),\quad \forall\,s\in\mathcal{S},\,a\in\mathcal{A},
\]
$\sum_{i=1}^{n_{S,A}}c_i=1$ and $c_i\geq 0$ ($i=1,\dots,n_{S,A}$).
\end{proposition}
The above proposition is implied by the proof of \cite[Theorem 5.1]{feinberg1996constrained}. For self-contained-ness, we also provide a simple proof by induction below.
\begin{proof}
Define the index of an arbitrary (randomized) policy $\pi$ as 
\[
\sum\nolimits_{s\in\mathcal{S}}(|\{a\in\mathcal{A}\,|\,\pi(a|s)>0\}|-1),
\] \ie, the difference between the total number of non-zero entries in $\pi$ (when viewed as a vector of length $SA$ or a matrix of size $S\times A$) and the total number of non-zero entries in a deterministic policy (\ie, the number of states). By definition, the index of a policy is at most $S(A-1)$. Below we prove the following claim, which immediately implies the desired result of Proposition \ref{convex_comb} by taking index equal to $S(A-1)$:
\begin{flushleft}
\textbf{Claim 1.} \emph{
For a policy with index $m$, there exist $m+1$ deterministic policies $\pi_1,\dots,\pi_{m+1}$ and nonnegative constants $c_1,\dots,c_{m+1}$, such that
\[
\pi(a|s)=\sum\nolimits_{i=1}^{m+1}c_i\pi_i(a|s),\quad \forall\,s\in\mathcal{S},\,a\in\mathcal{A},
\]
$\sum_{i=1}^{m+1}c_i=1$ and $c_i\geq 0$ ($i=1,\dots,m+1$).
}
\end{flushleft}

We prove this claim by induction on the index of $\pi$. 

\paragraph{Base step.} When the index of $\pi$ is $0$, the policy $\pi$ is deterministic, and hence we can simply take $c_1=1$ and $\pi_1=\pi$. 

\paragraph{Induction step.} Suppose that Claim 1 holds for index $m-1$ ($m\geq 1$). Then for a policy with index $m$, let 
\[
(s_{\min},a_{\min})\in\argmin\nolimits_{s\in\mathcal{S},\,a\in\mathcal{A},\,\pi(a|s)>0}\pi(a|s),
\]
and $\pi_{\min}=\pi(a_{\min}|s_{\min})\in(0,1)$. Note that $\pi_{\min}<1$ since otherwise the index would be $0$, which contradicts the assumption that $m\geq 1$. 

Now define $\pi_{m+1}$ as a deterministic policy such that $\pi(a_{\min}|s_{\min})=1$ and that for any $s\neq s_{\min}$, $\pi_{m+1}(a_s|s)=1$ for some (arbitrary) $a_s$ with  $\pi(a_s|s)>0$. Note that such a policy exists by the trivial fact that for any $s\in\mathcal{S}$, $\pi(a|s)>0$ for some $a\in\mathcal{A}$. By taking $c_{m+1}=\pi_{\min}\in(0,1)$, we can define a policy $\pi'$ with
\[
\pi'(a|s)=(\pi(a|s)-c_{m+1}\pi_{m+1})/(1-c_{m+1}),\quad \forall\,s\in\mathcal{S},\,a\in\mathcal{A}.
\]
It's easy to see that $\pi'$ is indeed a policy (\ie, $\pi'(a|s)\geq0$ for any $s\in\mathcal{S}$, $a\in\mathcal{A}$ and $\sum_{a\in\mathcal{A}}\pi'(a|s)=1$). In addition, by definition of $c_{m+1}$ and $\pi_{m+1}$, we also have 
\[
\{a\in\mathcal{A}\,|\,\pi'(a|s_{\min})>0\}\subseteq \{a\in\mathcal{A}\,|\,\pi(a|s_{\min})>0\}\backslash\{a_{\min}\}
\]
and 
\[
\{a\in\mathcal{A}\,|\,\pi'(a|s)>0\}\subseteq \{a\in\mathcal{A}\,|\,\pi(a|s)>0\},\quad \forall s\in\mathcal{S}, \,s\neq s_{\min},
\]
and hence the index of $\pi'$ is at most $m-1$. By the induction hypothesis, there exist $m$ deterministic policies $\pi_1,\dots,\pi_m$ and nonnegative constants $c_1',\dots,c_m'$, such that
\[
\pi'(a|s)=\sum\nolimits_{i=1}^{m}c_i'\pi_i(a|s),\quad \forall\,s\in\mathcal{S},\,a\in\mathcal{A},
\]
$\sum_{i=1}^{m}c_i'=1$ and $c_i'\geq 0$ ($i=1,\dots,m$), which immediately implies that
\[
\pi(a|s)=\sum\nolimits_{i=1}^{m+1}c_i\pi_i(a|s),\quad \forall\,s\in\mathcal{S},\,a\in\mathcal{A},
\]
with $c_i=(1-c_{m+1})c_i'$ ($i=1,\dots,m$). Since $c_i\geq 0$ ($i=1,\dots,m+1$) and $\sum_{i=1}^{m+1}c_i=1$ by definition, we have proved the claim for index $m$. By induction, this completes the proof.
\end{proof}

\begin{proof}[Proof of Proposition \ref{dobrushin_uniform}]
Let $\Pi_{\rm det}$ be the set of all deterministic policies $\pi$. 
By the finiteness of the state and action spaces, $\Pi_{\rm det}$ is also a finite set. For any $\pi\in\Pi_{\rm det}$, since $P_{\pi}$ is irreducible and aperiodic, there exists a positive integer $m_{\pi}$ such that $P_{\pi}^m$ is componentwisely positive for any $m\geq m_{\pi}$. Now by the finiteness of $\Pi_{\rm det}$, we can define $m_p=\max_{\pi\in\Pi_{\rm det}}m_{\pi}<\infty$, and then $P_{\pi}^m$ is componentwisely positive for any $\pi\in\Pi_{\rm det}$ and $m\geq m_p$. Accordingly, we also define 
\[
p_{\min}=\min_{\pi\in\Pi_{\rm det},\,s,\,s'\in\mathcal{S}}P_{\pi}^{m_p}(s,s')>0.
\] 

By Proposition \ref{convex_comb}, for any (randomized) policy $\pi$, there exist $n_{S,A}=S(A-1)+1$ policies $\pi_1,\dots,\pi_{n_{S,A}}\in\Pi_{\rm det}$ and nonnegative constants $c_1,\dots,c_{n_{S,A}}$, such that
\[
\pi(a|s)=\sum\nolimits_{i=1}^{n_{S,A}}c_i\pi_i(a|s),\quad \forall\,s\in\mathcal{S},\,a\in\mathcal{A},
\]
$\sum_{i=1}^{n_{S,A}}c_i=1$ and $c_i\geq 0$ ($i=1,\dots,n_{S,A}$).
By the linearity of $P_{\pi}$ in $\pi$, we have
\[
P_{\pi}=\sum\nolimits_{i=1}^{n_{S,A}}c_{\pi_i}P_{\pi_i}.
\]
This implies that for any $s,\,s'\in\mathcal{S}$, we have 
\BEQ\label{uniform_lb_P}
\begin{split}
P_{\pi}^{m_p}(s,s')&\geq \sum\nolimits_{i=1}^{n_{S,A}}c_{\pi_i}^{m_p}P_{\pi_i}^{m_p}(s,s')\geq 
p_{\min}\sum\nolimits_{i=1}^{n_{S,A}}c_{\pi_i}^{m_p}
\\
\text{(by convexity of $x^{m_p}$ for $x\geq 0$) }&\geq p_{\min}n_{S,A}(1/n_{S,A})^{m_p}=p_{\min}/n_{S,A}^{m_p-1}.
\end{split}
\EEQ

Accordingly, by Proposition \ref{dobrushin}, for any (randomized) policy $\pi:\mathcal{S}\rightarrow\mathcal{P}(\mathcal{A})$, there exists a constant 
\[
\alpha_{\pi}=1-S\min_{s,\,s'\in\mathcal{S}}P_{\pi}^{m_p}(s,s'), 
\] 
such that for any $r\geq 0$, 
\[
d_{\rm TV}(\rho (P_{\pi}^{m_p})^r,\mu_{\pi})\leq \alpha_{\pi}^r.
\]
By \eqref{uniform_lb_P}, we have $\alpha_{\pi}\in[0, 1-Sp_{\min}/n_{S,A}^{m_p-1}]\subseteq[0,1)$, which implies that 
\BEQ\label{d_TV_Pmp}
d_{\rm TV}(\rho (P_{\pi}^{m_p})^r,\mu_{\pi})\leq \tilde{\alpha}_{p,S,A}^r,
\EEQ
where $\tilde{\alpha}_{p,S,A}=1-Sp_{\min}/n_{S,A}^{m_p-1}\in[0,1)$. 


Now recall that we have 
\[
d_{\rm TV}(\mu P,\nu P)\leq d_{\rm TV}(\mu,\nu)
\] 
for any $\mu,\,\nu\in\mathcal{P}(\mathcal{S})$ and stochastic matrix $P$. 
Hence by \eqref{d_TV_Pmp} 
and by writing an arbitrary nonnegative integer $h$ as $h=rm_p+k$ ($0\leq k\leq m_p-1$), we have that  
for any (randomized) policy $\pi$, 
\[
d_{\rm TV}(\rho P_{\pi}^h,\mu_{{\pi}})\leq d_{\rm TV}(\rho (P_{\pi}^{m_p})^r,\mu_{{\pi}})\leq \tilde{\alpha}_{p,S,A}^r\leq C_{p,S,A}\alpha_{p,S,A}^{h},
\]
where $\alpha_{p,S,A}=\tilde{\alpha}_{p,S,A}^{1/m_p}$ and $C_{p,S,A}=1/\tilde{\alpha}_{p,S,A}$. This completes the proof.
\end{proof}

\subsection{Proofs of Lemmas \ref{error_bd_arb}, \ref{Vgamma-eta-pi}, \ref{VH-eta-pi} and  \ref{VH-eta}.}\label{proofs_subopt_lemmas}
\begin{proof}[Proof of Lemma \ref{error_bd_arb}]
By reorganization of the summations, we have 
\BEQ\label{Vgamma-repr}
\begin{split}
V^{\gamma}({\boldsymbol\pi})
&=(1-\gamma)\sum_{h=0}^{\infty}\gamma^h\sum_{s\in\mathcal{S},\,a\in\mathcal{A}}\prob^{\pi}(s_h=s,a_h=a|s_0\sim\rho)r(s,a)\\
&=(1-\gamma)\sum_{h=0}^{\infty}\gamma^h\sum_{s\in\mathcal{S},\,a\in\mathcal{A}}\left[\rho P_{\pi}^h\right]_s\pi(a|s)r(s,a)\\
&=\sum_{s\in\mathcal{S}}w^{\gamma}(s;\pi)\sum_{a\in\mathcal{A}}\pi(a|s)r(s,a),
\end{split}
\EEQ
where 
\[
\begin{split}
w^{\gamma}(s;\pi)&=(1-\gamma)\sum_{h=0}^{\infty}\left[\rho (\gamma P_{\pi})^h\right]_s=(1-\gamma)\left[\rho(I-\gamma P_{\pi})^{-1}\right]_s,\\
\end{split}
\]
 $\prob^{\pi}(s_h=s,a_h=a|s_0\sim\rho)$  is the probability of arriving at state $s$ and action $a$
in step $h$ starting from $s_0\sim \rho$ following policy $\pi$, and for a vector $x\in\reals^S$, we use $x_s$ or $[x]_s$ alternatively to denote its $s$-th element. Note that here $w^{\gamma}(s;\pi)=d_{\rho}^{\gamma,\pi}(s)$, and we use them alternatively throughout the appendix. In fact, for most of the time in the appendix, we use the former for simplicity and clarity (as $\rho$ is always fixed in our paper, while $\gamma$ may change) except for the final statements. 



Similarly, we also have 
\BEQ\label{VH_rewrite}
\begin{split}
V^H(\pi) 
&=\dfrac{1}{H}\sum_{h=0}^{H-1}\sum_{s\in\mathcal{S},\,a\in\mathcal{A}}\left[\rho P_\pi^h\right]_s\pi(a|s)r(s,a)=\sum_{s\in\mathcal{S}}w^H(s;\pi)\sum_{a\in\mathcal{A}}\pi(a|s)r(s,a),
\end{split}
\EEQ
where 
\[
w^H(s;\pi)=\dfrac{1}{H}\sum_{h=0}^{H-1}\left[\rho P_{\pi}^h\right]_s.
\]

Hence we have
\[
\begin{split}
\left|V^{\gamma}(\pi)-V^H(\pi)\right|
&\leq \sum_{s\in\mathcal{S}}|w^{\gamma}(s;\pi)-w^H(s;\pi)|\sum_{a\in\mathcal{A}}|r(s,a)|\pi(a|s)\\
&\leq R_{\max}\sum_{s\in\mathcal{S}}\left|(1-\gamma)\left[\rho(I-\gamma P_{\pi})^{-1}\right]_s-\dfrac{1}{H}\sum_{h=0}^{H-1}\left[\rho P_{\pi}^h\right]_s\right|\sum_{a\in\mathcal{A}}\pi(a|s)\\
&=R_{\max}\left\|\rho\left((1-\gamma)I-\dfrac{1}{H}\sum_{h=0}^{H-1}P_{\pi}^h(I-\gamma P_{\pi})\right)(I-\gamma P_{\pi})^{-1}\right\|_1.
\end{split}
\]

By Proposition \ref{dobrushin_uniform}, we have
\[
d_{\rm TV}(\rho P_{\pi}^h,\mu_{\pi})\leq C_{p,S,A}\alpha_{p,S,A}^h, 
\]
for some constant $\alpha_{p,S,A}\in[0,1)$ that depends only on $p,\,S,\,A$. 

Noticing that $1-\gamma=c/H$, we have
\[
\begin{split}
\rho\left((1-\gamma)I-\dfrac{1}{H}\sum_{h=0}^{H-1}P_{\pi}^h(I-\gamma P_{\pi})\right)
&=\left(1-\gamma-\dfrac{1}{H}\right)\rho-\dfrac{1}{H}\sum_{h=1}^{H-1}(1-\gamma)\rho P_{\pi}^h+\dfrac{1}{H}\gamma\rho P_{\pi}^H\\
&=\dfrac{1}{H}\left(\left(1-\dfrac{c}{H}\right)\rho P_{\pi}^H-\dfrac{c}{H}\sum_{h=1}^{H-1}\rho P_{\pi}^h\right)+\dfrac{c-1}{H}\rho\\
&=J_1+J_2,
\end{split}
\]
where 
\[
J_1=\dfrac{1}{H}\left(\dfrac{H-c}{H}(\rho P_{\pi}^H-\mu_{\pi})-\dfrac{c}{H}\sum_{h=1}^{H-1}(\rho P_{\pi}^h-\mu_{\pi})\right),\quad J_2 = \dfrac{c-1}{H}(\rho-\mu_{\pi}).
\]

Now using $\|\mu-\nu\|_1=2d_{\rm TV}(\mu,\nu)$ for any $\mu,\,\nu\in\mathcal{P}(\mathcal{S})$, we have
\[
\begin{split}
\left\|J_1\right\|_1&\leq \dfrac{1}{H}\left(\dfrac{H-c}{H}\|\rho P_{\pi}^H-\mu_{\pi}\|_1+\dfrac{c}{H}\sum_{h=1}^{H-1}\|\rho P_{\pi}^h-\mu_{\pi}\|_1\right)\\
&\leq \dfrac{2}{H}\left(\dfrac{H-c}{H}d_{\rm TV}(\rho P_{\pi}^H,\mu_{\pi})+\dfrac{c}{H}\sum_{h=1}^{H-1}d_{\rm TV}(\rho P_{\pi}^h,\mu_{\pi}))\right)\\
&\leq \dfrac{2C_{p,S,A}}{H}\left(\dfrac{H-c}{H}\alpha_{p,S,A}^H+\dfrac{c}{H}\sum_{h=1}^{H-1}\alpha_{p,S,A}^h\right)\\
&=\dfrac{2C_{p,S,A}}{H}\left(\dfrac{H-c}{H}\alpha_{p,S,A}^H+\dfrac{c}{H}\dfrac{\alpha_{p,S,A}-\alpha_{p,S,A}^H}{1-\alpha_{p,S,A}}\right)\\
&\leq \dfrac{2C_{p,S,A}(H-c)}{H^2}\alpha_{p,S,A}^H+\dfrac{2cC_{p,S,A}\alpha_{p,S,A}}{(1-\alpha_{p,S,A})H^2}.
\end{split}
\]

Similarly, we have
\[
\begin{split}
\|J_2(I-\gamma P_{\pi})^{-1}\|_1&=\dfrac{|c-1|}{H}\left\|\rho\sum_{h=0}^{\infty}\gamma^hP_{\pi}^h-\mu_{\pi}\sum_{h=0}^{\infty}\gamma^hP_{\pi}^h\right\|_1\\
&=\dfrac{|c-1|}{H}\left\|\sum_{h=0}^{\infty}\gamma^h \rho P_{\pi}^h-\sum_{h=0}^{\infty}\gamma^h\mu_{\pi}\right\|_1\leq\dfrac{|c-1|}{H}\sum_{h=0}^{\infty}\gamma^h\left\|\rho P_{\pi}^h-\mu_{\pi}\right\|_1\\
&=\dfrac{2|c-1|}{H}\sum_{h=0}^{\infty}\gamma^hd_{\rm TV}(\rho P_{\pi}^h,\mu_{\pi})\leq \dfrac{2|c-1|C_{p,S,A}}{(1-\gamma\alpha_{p,S,A})H}\leq \dfrac{2|c-1|C_{p,S,A}}{(1-\alpha_{p,S,A})H}.
\end{split}
\]

Finally, since we have 
\[
\|(I-\gamma P_{\pi})^{-1}\|_{\infty}=\left\|\sum_{h=0}^{\infty}(\gamma P_{\pi})^h\right\|_{\infty}\leq \sum_{h=0}^{\infty}\gamma^h=\dfrac{1}{1-\gamma}=H/c,
\]
we conclude that
\[
\begin{split}
|V^{\gamma}(\pi)-V^H(\pi)|&\leq R_{\max}\|J_1\|_1\|(I-\gamma P_{\pi})^{-1}\|_{\infty}+R_{\max}\|J_2(I-\gamma P_{\pi})^{-1}\|_1\\
&\leq2R_{\max}C_{p,S,A}\left(\dfrac{H-c}{cH}\alpha_{p,S,A}^H+\dfrac{\alpha_{p,S,A}+|c-1|}{(1-\alpha_{p,S,A})H}\right).
\end{split}
\]
Here we use the fact that for any row vector $x\in\reals^S$ and matrix $A\in\reals^{S\times S}$, 
\[
\|xA\|_1=\|A^Tx\|_1\leq \|A^T\|_1\|x\|_1=\|A\|_{\infty}\|x\|_1.
\]
This completes the proof. 
\end{proof}

\begin{proof}[Proof of Lemma \ref{Vgamma-eta-pi}]
The proof is similar to that of Lemma \ref{error_bd_arb}. In fact, by \eqref{Vgamma-repr} and \eqref{eta-def}, we have
\[
\begin{split}
|V^\gamma(\pi)-\eta(\pi)|&\leq R_{\max}\|(1-\gamma)[\rho(I-\gamma P_{\pi})^{-1}]-\mu_{\pi}\|_1\\
&=R_{\max}\left\|(1-\gamma)\sum_{h=0}^{\infty}\gamma^h\rho P_{\pi}^h-\mu_{\pi}\right\|_1=R_{\max}\left\|(1-\gamma)\sum_{h=0}^{\infty}\gamma^h(\rho P_{\pi}^h-\mu_{\pi})\right\|_1\\
&\leq \dfrac{2cR_{\max}}{H}\sum_{h=0}^{\infty}\gamma^hd_{\rm TV}(\rho P_{\pi}^h,\mu_{\pi})\\
\text{(by Proposition \ref{dobrushin_uniform}) } &\leq \dfrac{2cR_{\max}C_{p,S,A}}{H}\sum_{h=0}^{\infty}\gamma^h\alpha_{p,S,A}^h=\dfrac{2cR_{\max}C_{p,S,A}}{H}\dfrac{1}{1-\gamma\alpha_{p,S,A}}\leq \dfrac{2cR_{\max}C_{p,S,A}}{(1-\alpha_{p,S,A})H}.
\end{split}
\]
This completes the proof.
\end{proof}

\begin{proof}[Proof of Lemma \ref{VH-eta-pi}]
The key is to notice that we have 
\[
\begin{split}
\sum_{s\in\mathcal{S}}|w^H(s;\pi)-\mu_{\pi}(s)|&=\left\|\frac{1}{H}\sum_{h=0}^{H-1}[\rho P_{\pi}^h] - \mu_{\pi}\right\|_1\leq \frac{1}{H}\sum_{h=0}^{H-1}2C_{p,S,A}\alpha_{p,S,A}^h\leq \frac{2C_{p,S,A}}{H(1-\alpha_{p,S,A})}.
\end{split}
\]
Hence by \eqref{eta-def} and \eqref{VH_rewrite}, we have 
\[
\begin{split}
|V^H(\pi)-\eta(\pi)|&=\left|\sum_{s\in\mathcal{S}}w^H(s;\pi)\sum_{a\in\mathcal{A}}\pi(a|s)r(s,a)-\sum_{s\in\mathcal{S},\,a\in\mathcal{A}}\mu_{\pi}(s)\pi(a|s)r(s,a)\right|\\
&\leq \sum_{s\in\mathcal{S},a\in\mathcal{A}}\left|w^H(s;\pi)-\mu_{\pi}(s)\right|\pi(a|s)R_{\max}\leq \frac{2R_{\max}C_{p,S,A}}{H(1-\alpha_{p,S,A})}.
\end{split}
\]
This completes the proof.
\end{proof}

\begin{proof}[Proof of Lemma \ref{VH-eta}]
Let $L:\reals^S\rightarrow\reals^S$ be the Bellman operator, with
\[
[LJ]_s=\max_{a\in\mathcal{A}}\left(r(s,a)+\sum_{s'\in\mathcal{S}}p(s'|s,a)J(s')\right)
\]
for any $J\in\reals^S$. Then by the well-known dynamic programming principle \cite{puterman2014markov}, we have $V^{H,\star}=\frac{1}{H}\sum_{s\in\mathcal{S}}\rho_s[J^{H,\star}]_s$, where 
\[
J^{H,\star}=L^{H-1}r_{\max},
\]
where $r_{\max}\in\reals^S$ is defined by $r_{\max}(s)=\max_{a\in\mathcal{A}}r(s,a)$.

On the other hand, by the convergence property of value iteration algorithm for the infinite horizon average reward setting \cite[Proposition 8.5.1, Theorem 8.5.2]{puterman2014markov}, we have that for any nonnegative integers $r$ and $k$, 
\[
{\rm sp}(L^{rm_p+k+1}r_{\max}-L^{rm_p+k}r_{\max})\leq \tilde{\beta}_{p,S,A}^r{\rm sp}(L^{k+1}r_{\max}-L^kr_{\max}),
\]
where ${\rm sp}(J)$ is the span function defined as 
\[
{\rm sp}(J)=\max_{s\in\mathcal{S}} J(s)-\min_{s\in\mathcal{S}}J(s)
\] for any $J\in\reals^S$, $m_p$ is the positive integer defined in the proof of Proposition \ref{dobrushin_uniform}, and $\tilde{\beta}_{p,S,A}\in[0,1)$ is defined by
$\tilde{\beta}_{p,S,A}=1-Sp_{\min}$, 
where $p_{\min}>0$ is again defined as in the proof of Proposition \ref{dobrushin_uniform}.  Hence if we write $H-1$ as $H-1=rm_p+k$ for some nonnegative integers $r$ and $s$ with $0\leq k\leq m_p-1$, then we have
\[
{\rm sp}(LJ^{H,\star}-J^{H,\star})\leq \tilde{\beta}_{p,S,A}^r{\rm sp}(L^{k+1}r_{\max}-L^kr_{\max})\leq 4R_{\max}E_{p,S,A}m_p\beta_{p,S,A}^{H-1},
\]
where $\beta_{p,S,A}=\tilde{\beta}_{p,S,A}^{1/m_p}$, $E_{p,S,A}=1/\tilde{\beta}_{p,S,A}$, and we use the fact that for any $s\in\mathcal{S}$, 
\[
|[L^kr_{\max}](s)|\leq (k+1)R_{\max}.
\]

Finally, by \cite[Theorem 8.5.5]{puterman2014markov}, we have that for any $J\in\reals^S$,  
\[
\min_{s\in\mathcal{S}}[LJ-J]_s\leq\eta^\star\leq\max_{s\in\mathcal{S}}[LJ-J]_s,
\]
which immediately implies that 
\[
|\eta^\star-\sum\nolimits_{s\in\mathcal{S}}\rho_s[LJ-J]_s|\leq {\rm sp}(LJ-J).
\]
By plugging in $J=J^{H,\star}$ and noticing that 
\[
\sum_{s\in\mathcal{S}}\rho_sLJ^{H,\star}=\sum_{s\in\mathcal{S}}\rho_s[L^Hr_{\max}]_s=(H+1)V^{H+1,\star},
\] 
we have
\[
|\eta^\star-((H+1)V^{H+1,\star}-HV^{H,\star})|\leq 4R_{\max}E_{p,S,A}m_p\beta_{p,S,A}^{H-1},
\]
which implies that 
\[
\begin{split}
|(H+1)(V^{H+1,\star}-\eta^\star)|-|H(V^{H,\star}-\eta^\star)|&\leq
|(H+1)(V^{H+1,\star}-\eta^\star)-H(V^{H,\star}-\eta^\star)|\\
&\leq 4R_{\max}E_{p,S,A}m_p\beta_{p,S,A}^{H-1}.
\end{split}
\]
By telescoping the inequality from $1$ to $H-1$, we obtain that
\[
H|V^{H,\star}-\eta^\star|\leq |V^{1,\star}-\eta^\star|+4R_{\max}E_{p,S,A}m_p\sum_{h=1}^{H-1}\beta_{p,S,A}^{h-1}\leq 2R_{\max}+4R_{\max}E_{p,S,A}m_p\dfrac{\beta_{p,S,A}}{1-\beta_{p,S,A}},
\]
which shows that for any $H\geq 1$, 
\[
|V^{H,\star}-\eta^\star|\leq \dfrac{2R_{\max}D_{p,S,A}}{H},
\]
where $D_{p,S,A}=1+2E_{p,S,A}m_p\beta_{p,S,A}/(1-\beta_{p,S,A})$. This completes the proof. 
\end{proof}

\subsection{Proofs of Proposition \ref{kakade_prop}, Lemma \ref{grad_dom_average} and Lemma \ref{smooth_pg_ave}}
The proof of Proposition \ref{kakade_prop} follows the same steps as  \cite[Theorem 5.2]{agarwal2019theory}, with some modifications leading to a slightly tightened bound. For completeness, we provide a self-contained proof below.
\begin{proof}[Proof of Proposition \ref{kakade_prop}]
By \cite[Lemma C.1]{agarwal2019theory}, the policy gradient of $V^{\gamma}$ has the following form:
\BEQ\label{vgamma_grad}
\frac{\partial V^{\gamma}(\pi_{\theta})}{\partial \theta_{s,a}}=d_{\rho}^{\gamma,\pi_{\theta}}(s)\pi_{\theta}(a|s)A^{\gamma,\pi_{\theta}}(s,a),
\EEQ
and the gradient of the regularization term $\Omega$ has the form 
\BEQ\label{omega_grad}
\frac{\partial \Omega(\theta)}{\partial \theta_{s,a}}=\frac{\lambda}{SA}-\frac{\lambda}{S}\pi_{\theta}(a|s).
\EEQ

Now since $\nabla L^{\gamma}(\theta)\|_2\leq \lambda/(2SA)$, we have for any $s\in\mathcal{S}$ and $a\in\mathcal{A}$, 
\[
\frac{\partial L^{\gamma}(\theta)}{\partial \theta_{s,a}}=\frac{1}{1-\gamma}d_{\rho}^{\gamma,\pi_{\theta}}(s)\pi_{\theta}(a|s)A^{\gamma,\pi_{\theta}}(s,a)+\frac{\lambda}{SA}-\frac{\lambda}{S}\pi_{\theta}(a|s)\leq \lambda /(2SA),
\]
from which we see that 
\[
A^{\gamma,\pi_{\theta}}(s,a)\leq \frac{\lambda(1-\gamma)}{d_{\rho}^{\gamma,\pi_{\theta}}(s)}\left(\frac{1}{S}-\frac{1}{2SA\pi_{\theta}(a|s)}\right)\leq\frac{\lambda(1-\gamma)}{Sd_{\rho}^{\gamma,\pi_{\theta}}(s)}.
\]

Now notice that for any stationary policy $\pi$ and any state $s\in\mathcal{S}$,
\[
\begin{split}
d_{\rho}^{\gamma,\pi}(s)&=w^{\gamma}(s;\pi)=(1-\gamma)\sum_{h=0}^{\infty}\gamma^h [\rho P_{\pi}^h]_s\geq (1-\gamma)\sum_{k=0}^{\infty}\gamma^{km_p} [\rho P_{\pi}^{km_p}]_s\\
&\geq (1-\gamma)\sum_{k=0}^{\infty}\gamma^{km_p}p_{\min}/n_{S,A}^{m_p-1}=\frac{(1-\gamma)p_{\min}}{(1-\gamma^{m_p})n_{S,A}^{m_p-1}},
\end{split}
\]
where the quantities $m_p$, $p_{\min}$ and $n_{S,A}$ are defined in Proposition \ref{dobrushin_uniform}. In addition, by definition, we also have 
$d_{\rho}^{\gamma,\pi}(s)\geq (1-\gamma)\rho(s)$ for any stationary policy $\pi$ and state $s\in\mathcal{S}$. 

Finally, by the performance difference lemma in the discounted setting \cite[Lemma 3.2]{agarwal2019theory}, we have
\[
\begin{split}
V^{\gamma,\star}-V^{\gamma}(\pi_{\theta})&=\sum_{s\in\mathcal{S}}d_{\rho}^{\gamma,\pi^{\gamma,\star}}(s)\sum_{a\in\mathcal{A}}\pi^{\gamma,\star}(a|s)A^{\gamma,\pi_{\theta}}(s,a)\\
&\leq \frac{\lambda(1-\gamma)}{S}\sum_{s\in\mathcal{S}}\min\left\{\frac{d_{\rho}^{\gamma,\pi^\gamma,\star}(s)}{(1-\gamma)\rho(s)},
\frac{d_{\rho}^{\gamma,\pi^\gamma,\star}(s)n_{S,A}^{m_p-1}(1-\gamma^{m_p})}{(1-\gamma)p_{\min}}\right\}\\
&\leq \lambda\left\{\left\|\frac{d_{\rho}^{\gamma,\pi^{\gamma,\star}}}{\rho}\right\|_{\infty}, \frac{S\|d_{\rho}^{\gamma,\pi^{\gamma,\star}}\|_{\infty}}{1-\alpha_{p,S,A}}\right\},
\end{split}
\]
where the last step uses the fact that 
\[
\alpha_{p,S,A}=\tilde{\alpha}_{p,S,A}^{1/m_p}\geq \tilde{\alpha}_{p,S,A}=1-Sp_{\min}/n_{S,A}^{m_p-1},
\] 
which comes from the proof of Proposition \ref{dobrushin_uniform}. This completes the proof. 
\end{proof}

Proof of Lemma \ref{grad_dom_average} relies on the following lemma.
\begin{lemma}[\cite{even2009online, gergely2010online}. Average reward performance difference lemma]
Suppose that Assumption \ref{finite-ergodic} holds. Then we have
\BEQ\label{perf_diff_lemma}
\eta(\pi)-\eta(\pi')=\sum_{s\in\mathcal{S}}\mu_{\pi}(s)\sum_{a\in\mathcal{A}}\pi(a|s)\bar{A}^{\pi'}(s,a).
\EEQ
\end{lemma}

\begin{proof}[Proof of Lemma \ref{grad_dom_average}]
By the well-known policy gradient theorem \cite{sutton2000policy} and some simplification, we have  
\BEQ\label{grad_eta_formula}
\frac{\partial\eta(\pi_{\theta})}{\partial \theta_{s,a}}=\mu_{\pi_{\theta}}(s)\pi_{\theta}(a|s)\bar{A}^{\pi_{\theta}}(s,a).
\EEQ
Now since $\|\nabla_{\theta}\bar{L}(\theta)\|_2\leq \lambda/(2SA)$, recalling the form of $\nabla_{\theta} \Omega(\theta)$ in the proof of Proposition \ref{kakade_prop}, we have that for any $s\in\mathcal{S}$ and $a\in\mathcal{A}$, 
\BEQ\label{Lsa_bound}
\frac{\partial \bar{L}(\theta)}{\partial\theta_{s,a}}=\mu_{\pi_{\theta}}(s)\pi_{\theta}(a|s)\bar{A}^{\pi_{\theta}}(s,a)+\frac{\lambda}{SA}-\frac{\lambda}{S}\pi_{\theta}(a|s)\leq \lambda/(2SA).
\EEQ
Hence we have 
\[
\bar{A}^{\pi_{\theta}}(s,a)\leq \frac{1}{\mu_{\pi_{\theta}}(s)}\left(\frac{\lambda}{S}-\frac{\lambda}{2SA\pi_{\theta}(a|s)}\right)\leq \frac{\lambda}{\mu_{\pi_{\theta}}(s)S}.
\]
Now notice that for any stationary policy $\pi$, $\mu_{\pi}P_{\pi}=\mu_{\pi}$ and hence $\mu_{\pi}(s)=\sum_{s'\in\mathcal{S}}\mu_{\pi}(s')P_{\pi}^{m_p}(s',s)\geq p_{\min}/n_{S,A}^{m_p-1}$, where the quantities $\mu_{\pi}$, $m_p$, $p_{\min}$ and $n_{S,A}$  are defined in Proposition \ref{dobrushin_uniform}. 

Finally, we have 
\[
\begin{split}
\eta^\star-\eta(\pi_{\theta})&=\sum_{s\in\mathcal{S}}\mu_{\pi^\star}(s)\sum_{a\in\mathcal{A}}\pi^\star(a|s)\bar{A}^{\pi_{\theta}}(s,a)\\
&\leq\frac{\lambda}{S} \sum_{s\in\mathcal{S}}\frac{\mu_{\pi^\star}(s)}{\mu_{\pi_{\theta}}(s)}\leq \lambda \frac{\|\mu_{\pi^\star}\|_\infty n_{S,A}^{m_p-1}}{p_{\min}}\leq \lambda \frac{S\|\mu_{\pi^\star}\|_\infty}{1-\alpha_{p,S,A}},
\end{split}
\]
where again we use the fact that 
\[
\alpha_{p,S,A}=\tilde{\alpha}_{p,S,A}^{1/m_p}\geq \tilde{\alpha}_{p,S,A}=1-Sp_{\min}/n_{S,A}^{m_p-1}.
\] 
This completes the proof.
\end{proof}

Now we are ready to show the proof of Lemma \ref{smooth_pg_ave}.
\begin{proof} [Proof of Lemma \ref{smooth_pg_ave}] We prove a slightly generalized version of the claimed results assuming only that $|r(s,a)|\leq R_{\max}$ 
instead of $r(s,a)\in[0,1]$ as in Assumption \ref{setting}. 

Firstly, we show that 
\BEQ\label{mu-lipschitz}
\|\mu_{\pi_{\theta_1}}-\mu_{\pi_{\theta_2}}\|_1\leq \frac{2\sqrt{S}C_{p,S,A}}{1-\alpha_{p,S,A}}\|\theta_1-\theta_2\|_2.
\EEQ To see this, first notice that \cite{haviv1984perturbation}
\BEQ\label{perturbation_existing}
\mu_{\pi_{\theta_1}}-\mu_{\pi_{\theta_2}} = \mu_{\pi_{\theta_1}}(P_1-P_2)Y_2,
\EEQ
where $P_i=P_{\pi_{\theta_i}}$ and $Y_i=\sum_{h=0}^{\infty}(P_i^h-P_i^{\infty})$, with ${\bf 1}\in\reals^S$ being the all-one vector and $P_i^{\infty}={\bf 1}\mu_{\pi_{\theta_i}}=\lim_{h\rightarrow\infty}P_i^h$ ($i=1,2$). 

By Proposition \ref{dobrushin_uniform}, we have that for any policy $\pi$,
\[
\|e_j P_{\pi}^h-\mu_{\pi}\|_1\leq 2C_{p,S,A}\alpha_{p,S,A}^h,
\]
where $e_j$ is the coordinate vector with $1$ in the $j$-th coordinate and $0$ elsewhere. Hence 
\[
\|P_{\pi}^h-P_{\pi}^{\infty}\|_\infty\leq \max_{j=1,\dots,S}\|e_j P_{\pi}^h-\mu_{\pi}\|_1\leq 2C_{p,S,A}\alpha_{p,S,A}^h,
\]
where $P_{\pi}^{\infty}={\bf 1}\mu_{\pi}$, which implies that for $i=1,2$, 
\[
\|Y_i\|_\infty\leq 2C_{p,S,A}\sum_{h=0}^{\infty}\alpha_{p,S,A}^h\leq \frac{2C_{p,S,A}}{1-\alpha_{p,S,A}}.
\]

Hence we have 
\[
\|\mu_{\pi_{\theta_1}}-\mu_{\pi_{\theta_2}}\|_1\leq \|P_1-P_2\|_\infty\|Y_2\|_\infty\|\mu_{\pi_{\theta_1}}\|_1\leq \frac{2C_{p,S,A}}{1-\alpha_{p,S,A}}\|P_1-P_2\|_\infty.
\]
By noticing that 
\[
\begin{split}
\|P_1-P_2\|_\infty&= \max_{s\in\mathcal{S}}\sum_{s'\in\mathcal{S}}\left|\sum_{a\in\mathcal{A}}p(s'|s,a)(\pi_{\theta_1}(a|s)-\pi_{\theta_2}(a|s))\right|\\
&\leq \max_{s\in\mathcal{S}}\sum_{a\in\mathcal{A}}\sum_{s'\in\mathcal{S}}p(s'|s,a)|\pi_{\theta_1}(a|s)-\pi_{\theta_2}(a|s)|\\
&\leq \max_{s\in\mathcal{S}}\|\pi_{\theta_1}(\cdot|s)-\pi_{\theta_2}(\cdot|s)\|_1\leq \sqrt{S}\|\theta_1-\theta_2\|_2,
\end{split}
\]
we obtain \eqref{mu-lipschitz}. Here the last step uses \cite[Proposition 4]{gao2017properties} (soft-max function is $1$-Lipschitz in $\ell_2$-norm) and the fact that $\|x\|_1\leq \sqrt{S}\|x\|_2$ for any $x\in\reals^S$. 

Secondly, we show that for any $s\in\mathcal{S}$ and $a\in\mathcal{A}$, we have
\BEQ\label{Abar-lipschitz}
|\bar{A}^{\pi_{\theta_1}}(s,a)-\bar{A}^{\pi_{\theta_2}}(s,a)|\leq 2R_{\max}\sqrt{S}\|\theta_1-\theta_2\|_2\left(\left(\frac{2C_{p,S,A}}{1-\alpha_{p,S,A}}+1\right)^3+\frac{7C_{p,S,A}}{1-\alpha_{p,S,A}}+2\right).
\EEQ
To see this, first notice that for $i=1,2$, 
\[
\begin{split}
\bar{Q}^{\pi_{\theta_i}}(s,a)&=r(s,a)-\eta(\pi_{\theta_i})\\
&\quad+\sum_{h=1}^{\infty}\left(\sum_{s\in\mathcal{S},a\in\mathcal{A}}[\rho_{1} P_i^{h-1}]_s\pi_{\theta_i}(a|s)r(s,a)-\sum_{s\in\mathcal{S},a\in\mathcal{A}}\mu_{\pi_{\theta_i}}(s)\pi_{\theta_i}(a|s)r(s,a)\right)\\
&=r(s,a)-\eta(\pi_{\theta_i})+\sum_{s\in\mathcal{S},a\in\mathcal{A}}\pi_{\theta_i}(a|s)r(s,a)\sum_{h=0}^{\infty}([\rho_{1} P_i^h]_s-\mu_{\pi_{\theta_i}}(s))\\
&=r(s,a)-\eta(\pi_{\theta_i})+\sum_{s\in\mathcal{S},a\in\mathcal{A}}\pi_{\theta_i}(a|s)r(s,a)[\rho_{1} Y_i]_s, 
\end{split}
\]
where $\rho_{1}(s')=p(s'|s,a)$ for any $s'\in\mathcal{S}$. This implies that 
\[
\begin{split}
|\bar{Q}^{\pi_{\theta_1}}(s,a)-\bar{Q}^{\pi_{\theta_2}}(s,a)|&\leq |\eta(\pi_{\theta_1})-\eta(\pi_{\theta_2})|\\
&\quad +\sum_{s\in\mathcal{S},a\in\mathcal{A}}|r(s,a)|\left|\pi_{\theta_1}(a|s)[\rho_{1}Y_1]_s-\pi_{\theta_2}(a|s)[\rho_{1}Y_2]_s\right|\\
&\leq |\eta(\pi_{\theta_1})-\eta(\pi_{\theta_2})|+R_{\max}\sum_{s\in\mathcal{S},a\in\mathcal{A}}\pi_{\theta_1}(a|s)|[\rho_{1} Y_1]_s-[\rho_{1}Y_2]_s|\\
&\quad + R_{\max}\sum_{s\in\mathcal{S},a\in\mathcal{A}}|\pi_{\theta_1}(a|s)-\pi_{\theta_2}(a|s)||[\rho_{1}Y_2]_s|\\
&\leq |\eta(\pi_{\theta_1})-\eta(\pi_{\theta_2})| + R_{\max}\|\rho_{1} Y_1-\rho_{1}Y_2\|_1\\
&\quad +R_{\max}\sum_{s\in\mathcal{S}}|[\rho_{1}Y_2]_s|\|\pi_{\theta_1}(\cdot|s)-\pi_{\theta_2}(\cdot|s)\|_1\\
&\leq |\eta(\pi_{\theta_1})-\eta(\pi_{\theta_2})| + R_{\max}\|Y_1-Y_2\|_\infty + R_{\max}\sqrt{S}\|\theta_1-\theta_2\|_2\|Y_2\|_\infty.
\end{split}
\]

We now bound each of the three terms on the right-hand side. 

Firstly, by \eqref{eta-def}, we have 
\[
\begin{split}
|\eta(\pi_{\theta_1})-\eta(\pi_{\theta_2})|&\leq \sum_{s\in\mathcal{S},a\in\mathcal{A}}|r(s,a)||\mu_{\pi_{\theta_1}}(s)\pi_{\theta_1}(a|s)-\mu_{\pi_{\theta_2}}(s)\pi_{\theta_2}(a|s)|\\
&\leq R_{\max}\sum_{s\in\mathcal{S},a\in\mathcal{A}}(\pi_{\theta_1}(a|s)|\mu_{\pi_{\theta_1}}(s)-\mu_{\pi_{\theta_2}}(s)|+|\pi_{\theta_1}(a|s)-\pi_{\theta_2}(a|s)|\mu_{\pi_{\theta_2}}(s))\\
&\leq R_{\max}\|\mu_{\pi_{\theta_1}}-\mu_{\pi_{\theta_2}}\|_1+R_{\max}\sqrt{S}\|\theta_1-\theta_2\|_2\\
&\leq R_{\max}\sqrt{S}\left(\frac{2C_{p,S,A}}{1-\alpha_{p,S,A}}+1\right)\|\theta_1-\theta_2\|_2.
\end{split}
\]

Next, notice that for $i=1,2$, we have \cite{haviv1984perturbation}
\[
Y_i=(I-P_i+P_i^{\infty})^{-1}-P_i^{\infty}.
\]
Hence we have 
\[
\|(I-P_i+P_i^{\infty})^{-1}\|_\infty\leq \|Y_i\|_\infty + \|P_i^{\infty}\|_\infty\leq \frac{2C_{p,S,A}}{1-\alpha_{p,S,A}}+1.
\]
%

Now by Banach perturbation lemma \cite[III.2.2, Theorem 2.5]{stewart1990matrix}, we have 
\[
\begin{split}
\|Y_1-Y_2\|_\infty&\leq \|(I-P_1+P_1^{\infty})^{-1}\|_\infty \|(I-P_2+P_2^{\infty})^{-1}\|_\infty\|P_2-P_1+P_1^{\infty}-P_2^{\infty}\|_\infty\\
&\quad+\|P_1^{\infty}-P_2^{\infty}\|_\infty\\
&\leq 
\left(\frac{2C_{p,S,A}}{1-\alpha_{p,S,A}}+1\right)^2
\|P_1-P_2\|_\infty+\left(\left(\frac{2C_{p,S,A}}{1-\alpha_{p,S,A}}+1\right)^2
+1\right)\|P_1^{\infty}-P_2^{\infty}\|_\infty\\
&\leq \sqrt{S}\left(\frac{2C_{p,S,A}}{1-\alpha_{p,S,A}}+1\right)^2
\|\theta_1-\theta_2\|_2+
\left(\left(\frac{2C_{p,S,A}}{1-\alpha_{p,S,A}}+1\right)^2+1\right)
\|\mu_{\pi_{\theta_1}}-\mu_{\pi_{\theta_2}}\|_1\\
&\leq \sqrt{S}\left(\left(\frac{2C_{p,S,A}}{1-\alpha_{p,S,A}}+1\right)^3
+\frac{2C_{p,S,A}}{1-\alpha_{p,S,A}}\right)\|\theta_1-\theta_2\|_2.
\end{split}
\]
Putting these together, we obtain that 
\[
\begin{split}
|\bar{Q}^{\pi_{\theta_1}}(s,a)-\bar{Q}^{\pi_{\theta_2}}(s,a)|&\leq R_{\max}\sqrt{S}\|\theta_1-\theta_2\|_2\left(\left(\frac{2C_{p,S,A}}{1-\alpha_{p,S,A}}+1\right)^3+\frac{6C_{p,S,A}}{1-\alpha_{p,S,A}}+1\right).
\end{split}
\]

By the fact that $\bar{V}^{\pi}(s)=\Expect_{a\sim\pi(\cdot|s)}\bar{Q}^\pi(s,a)$, we also have for any $s\in\mathcal{S}$, 
\[
\begin{split}
|\bar{V}^{\pi_{\theta_1}}(s)-\bar{V}^{\pi_{\theta_2}}(s)|&\leq \sum_{a\in\mathcal{A}}|\pi_{\theta_1}(a|s)\bar{Q}^{\pi_{\theta_1}}(s,a)-\pi_{\theta_2}(a|s)\bar{Q}^{\pi_{\theta_2}}(s,a)|\\
&\leq \sum_{a\in\mathcal{A}}\pi_{\theta_1}(a|s)|\bar{Q}^{\pi_{\theta_1}}(s,a)-\bar{Q}^{\pi_{\theta_2}}(s,a)|\\
&\quad+\sum_{a\in\mathcal{A}}|\pi_{\theta_1}(a|s)-\pi_{\theta_2}(a|s)||\bar{Q}^{\pi_{\theta_2}}(s,a)|\\
&\leq R_{\max}\sqrt{S}\|\theta_1-\theta_2\|_2\left(\left(\frac{2C_{p,S,A}}{1-\alpha_{p,S,A}}+1\right)^3+\frac{8C_{p,S,A}}{1-\alpha_{p,S,A}}+3\right),\\
\end{split}
\]
where the last step uses the fact that for $i=1,2$ and any $s\in\mathcal{S}$ and $a\in\mathcal{A}$, 
\BEQ\label{Qbar_ub}
\begin{split}
|\bar{Q}^{\pi_{\theta_i}}(s,a)|&\leq |r(s,a)|+|\eta(\pi_{\theta_i})|+\sum_{s\in\mathcal{S},a\in\mathcal{A}}\pi_{\theta_i}(a|s)|r(s,a)|
|[\rho_1 Y_i]_s|\\
&\leq 2R_{\max}+R_{\max}\|\rho_1 Y_i\|_1\leq 2R_{\max}(1+C_{p,S,A}/(1-\alpha_{p,S,A})).
\end{split}
\EEQ
These immediately imply \eqref{Abar-lipschitz}. Note that again by $\bar{V}^{\pi}(s)=\Expect_{a\sim\pi(\cdot|s)}\bar{Q}^\pi(s,a)$, \eqref{Qbar_ub} also holds when $\bar{Q}^{\pi_{\theta_i}}(s,a)$ is replaced with $\bar{V}^{\pi_{\theta_i}}(s)$. 

Finally, combining \eqref{grad_eta_formula}, \eqref{mu-lipschitz}, \eqref{Abar-lipschitz}, \eqref{Qbar_ub} and the fact that the soft-max function is $1$-Lipschitz in $\ell_2$-norm, we have 
\[
\begin{split}
\|\nabla_{\theta}\eta(\pi_{\theta_1})&-\nabla_{\theta}\eta(\pi_{\theta_2})\|_2\leq \|\nabla_{\theta}\eta(\pi_{\theta_1})-\nabla_{\theta}\eta(\pi_{\theta_2})\|_1=\sum_{s\in\mathcal{S},a\in\mathcal{A}}\left|\frac{\partial \eta(\pi_{\theta_1})}{\partial \theta_{s,a}}-\frac{\partial \eta(\pi_{\theta_2})}{\partial \theta_{s,a}}\right|\\
&\leq\sum_{s\in\mathcal{S},a\in\mathcal{A}} (\mu_{\pi_{\theta_1}}(s)\pi_{\theta_1}(a|s)\left|\bar{A}^{\pi_{\theta_1}}(s,a)-\bar{A}^{\pi_{\theta_2}}(s,a)\right|\\
&\quad +\mu_{\pi_{\theta_1}}(s)\left|\pi_{\theta_1}(a|s)-\pi_{\theta_2}(a|s)\right||\bar{A}^{\pi_{\theta_2}}(s,a)|\\
&\quad +\left|\mu_{\pi_{\theta_1}}(s)-\mu_{\pi_{\theta_2}}(s)\right|\pi_{\theta_2}(a|s)|\bar{A}^{\pi_{\theta_2}}(s,a)|)\\
&\leq 2R_{\max}\sqrt{S}\|\theta_1-\theta_2\|_2\left(\left(\frac{2C_{p,S,A}}{1-\alpha_{p,S,A}}+1\right)^3+\frac{7C_{p,S,A}}{1-\alpha_{p,S,A}}+2\right)\\
&\quad + 4R_{\max}\left(1+\frac{C_{p,S,A}}{1-\alpha_{p,S,A}}\right)\sqrt{S}\|\theta_1-\theta_2\|_2\\
&\quad +4R_{\max}\left(1+\frac{C_{p,S,A}}{1-\alpha_{p,S,A}}\right)\frac{2\sqrt{S}C_{p,S,A}}{1-\alpha_{p,S,A}}\|\theta_1-\theta_2\|_2\\
&=2R_{\max}\sqrt{S}\|\theta_1-\theta_2\|_2\left(\left(\frac{2C_{p,S,A}}{1-\alpha_{p,S,A}}+1\right)^3+4\left(\frac{C_{p,S,A}}{1-\alpha_{p,S,A}}\right)^2+\frac{13C_{p,S,A}}{1-\alpha_{p,S,A}}+4\right)\\
&\leq 22R_{\max}\sqrt{S}\left(\frac{2C_{p,S,A}}{1-\alpha_{p,S,A}}+1\right)^3\|\theta_1-\theta_2\|_2,
\end{split}
\]
where the last step uses the fact that $2C_{p,S,A}/(1-\alpha_{p,S,A})>1$. 

Finally, noticing that $\Omega(\theta)$ is $\frac{2\lambda}{S}$-strongly smooth \cite[Lemma D.4]{agarwal2019theory}, the proof is finished.
\end{proof}

\section{Proofs for DAE REINFORCE algorithm}
In this section we provide the proofs for the convergence result of DAE REINFORCE algorithm.


\subsection{Proof of Lemma \ref{spinup_grad_error_lemma}}

In this section, we prove a slightly generalized version of Lemma \ref{spinup_grad_error_lemma}, which may be useful for future research. 
\begin{lemma}[Slight generalization of Lemma \ref{spinup_grad_error_lemma}]
Suppose that Assumption \ref{finite-ergodic} holds. In addition, suppose that the policy parametrization is such that $\|\nabla_{\theta}\log\pi_{\theta}(a|s)\|_2\leq \tilde{C}$ for any $\theta\in\Theta$, $s\in\mathcal{S}$ and $a\in\mathcal{A}$ and $\Omega(\theta)$ is differentiable. Then we have the following gradient estimation error:
\BEQ\label{spinup_grad_error_generalized}
\begin{split}
\left\|\Expect_k[\hat{g}_k]-\nabla \bar{L}(\theta^k)\right\|_2&\leq \frac{8\tilde{C}R_{\max}C_{p,S,A}}{\lfloor \beta H\rfloor (1-\alpha_{p,S,A})}\left(1+\frac{C_{p,S,A}}{1-\alpha_{p,S,A}}\right)\\
&\qquad+4\tilde{C}R_{\max}C_{p,S,A}\frac{1-\gamma}{(1-\alpha_{p,S,A})^2}\\
&\qquad+2\tilde{C}R_{\max}\gamma^{(1-\beta)H}\left(1+\frac{C_{p,S,A}}{1-\alpha_{p,S,A}}\right).
\end{split}
\EEQ
Here the constants $C_{p,S,A}>1$ and $\alpha_{p,S,A}\in[0,1)$ are defined in Proposition \ref{dobrushin_uniform}. In particular, when Assumptions \ref{finite-ergodic} and \ref{setting} hold, $\Omega(\theta)$ is obviously differentiable and we have $\tilde{C}=2$ and $R_{\max}=1$. 
\end{lemma}

\begin{proof}[Proof of Lemma \ref{spinup_grad_error_lemma}]
By the well-known policy gradient theorem \cite{sutton2000policy}, we have
\[
\nabla_{\theta}\eta(\pi_{\theta})=\sum_{s\in\mathcal{S}}\mu_{\pi_{\theta}}(s)\sum_{a\in\mathcal{A}}\nabla_{\theta}\pi_{\theta}(a|s)\bar{A}^{\pi_{\theta}}(s,a).
\]
In addition, by \cite[Lemma 4.10]{agarwal2019reinforcement} we have that for any function $f:\mathcal{S}\rightarrow\reals$ independent of the trajectories $\tau_i$ ($i=1,\dots,N$), 
\[
\Expect_k[\nabla_{\theta}\log\pi_{\theta^k}(a_h^i|s_h^i)f(s_h^i)]=0.
\]
Hence by first taking $f=b$ and then $f=V^{\gamma,\pi_{\theta^k}}$, 
we have (for an arbitrary $i=1,\dots,N$)
\[
\begin{split}
\Expect_k[\hat{g}_k]-\nabla_{\theta}\Omega(\theta_k)=J_1-J_2,
\end{split}
\]
where 
\[
\begin{split}
J_1&=\Expect_k\left[\frac{1}{\lfloor \beta H\rfloor}\sum_{h=0}^{\lfloor\beta H\rfloor-1}\nabla_{\theta}\log\pi_{\theta^k}(a_h^i|s_h^i)\left(\Expect_k\left[\sum_{h'=h}^{\infty}\gamma^{h'-h}r_{h'}^i\Big|s_h^i,a_h^i,\pi_{\theta^k}\right]-\frac{1}{1-\gamma}V^{\gamma,\pi_{\theta^k}}(s_h^i)\right)\right]\\
&=\sum_{s\in\mathcal{S}}w^{\lfloor \beta H\rfloor}(s;\pi_{\theta^k})\sum_{a\in\mathcal{A}}\pi_{\theta^k}(a|s)\nabla_{\theta}\log\pi_{\theta^k}(a|s)\left(\frac{1}{1-\gamma}Q^{\gamma,\pi_{\theta^k}}(s,a)-\frac{1}{1-\gamma}V^{\gamma,\pi_{\theta^k}}(s)\right),
\end{split}
\]
and 
\[
\begin{split}
J_2&=\Expect_k\left[\frac{1}{\lfloor \beta H\rfloor}\sum_{h=0}^{\lfloor\beta H\rfloor-1}\nabla_{\theta}\log\pi_{\theta^k}(a_h^i|s_h^i)\Expect_k\left[\sum_{h'=H}^{\infty}\gamma^{h'-h}r_{h'}^i\Big|s_h^i,a_h^i,\pi_{\theta^k}\right]\right].\\
\end{split}
\]

Let's first consider $J_1$. Notice that we have
\[
\begin{split}
\frac{1}{1-\gamma}Q^{\gamma,\pi_{\theta^k}}(s,a)-\frac{1}{1-\gamma}V^{\gamma,\pi_{\theta^k}}(s) 
=&\,\Expect_k\left[\sum_{h'=h}^{\infty}\gamma^{h'-h}(r_{h'}^i-\eta(\pi_{\theta^k}))\Big|s_h^i=s,a_h^i=a,\pi_{\theta^k}\right]\\
&\,-\Expect_k\left[\sum_{h'=h}^{\infty}\gamma^{h'-h}(r_{h'}^i-\eta(\pi_{\theta^k}))\Big|s_h^i=s,\pi_{\theta^k}\right]\\
&=I_1(s,a)-I_2(s),
\end{split}
\]
where 
\[
\begin{split}
I_1(s,a)&=
\sum_{h'=h}^{\infty}\gamma^{h'-h}\left(\sum_{s'\in\mathcal{S},a'\in\mathcal{A}}\prob^{\pi_{\theta^k}}(s_{h'}^i=s',a_{h'}^i=a'|s_h^i=s,a_h^i=a)r(s,a)-\eta(\pi_{\theta^k})\right)\\
I_2(s)&=
\sum_{h'=h}^{\infty}\gamma^{h'-h}\left(\sum_{s'\in\mathcal{S},a'\in\mathcal{A}}\prob^{\pi_{\theta^k}}(s_{h'}^i=s',a_{h'}^i=a'|s_h^i=s)r(s,a)-\eta(\pi_{\theta^k})\right)
\end{split}
\]
By writing out the conditional expectations explicitly, we have
\[
\begin{split}
I_1(s,a)=r(s,a)-\eta(\pi_{\theta^k})+\gamma
\sum_{h'=h+1}^{\infty}\gamma^{h'-h-1}&\left(\sum_{s'\in\mathcal{S},a'\in\mathcal{A}}[\rho_{h+1,i} P_{\pi_{\theta^k}}^{h'-h-1}]_{s'}\pi_{\theta^k}(a'|s')r(s',a')\right.\\
&\left.-\sum_{s'\in\mathcal{S},a'\in\mathcal{A}}\mu_{\pi_{\theta^k}}(s')\pi_{\theta^k}(a'|s')r(s',a')\right)
\end{split}
\]
where $\rho_{h+1,i}(s'')=\prob^{\pi_{\theta^k}}(s_{h+1}^i=s''|s_h^i=s,a_h^i=a)=p(s''|s,a)$ for any $s''\in\mathcal{S}$. 

Hence by Proposition \ref{dobrushin_uniform}, we have 
\[
|I_1(s,a)|\leq 2R_{\max}+2C_{p,S,A}R_{\max}\frac{\gamma}{1-\alpha_{p,S,A}\gamma}\leq 2R_{\max}+\frac{2C_{p,S,A}R_{\max}}{1-\alpha_{p,S,A}}.
\]
In addition, noticing that 
\[
\bar{Q}^{\pi_{\theta^k}}(s,a)=\Expect_k\left[\sum_{h'=h}^{\infty}r(s_{h'}^i,a_{h'}^i)-\eta(\pi_{\theta^k})\Big|s_h^i=s,a_h^i=a,\pi_{\theta^k}\right],
\]
we also have 
\[
\begin{split}
|I_1(s,a)-\bar{Q}^{\pi_{\theta^k}}(s,a)|&\leq  \sum_{h'=h+1}^{\infty}(1-\gamma^{h'-h})\sum_{s'\in\mathcal{S},a'\in\mathcal{A}}\left|[\rho_{h+1,i} P_{\pi_{\theta^k}}^{h'-h-1}]_{s'}-\mu_{\pi_{\theta^k}}(s')\right|\pi_{\theta^k}(a'|s')|r(s',a')|\\
&\leq 2C_{p,S,A}R_{\max}\sum_{h'=h+1}^{\infty}(1-\gamma^{h-h'})\alpha_{p,S,A}^{h'-h-1}\\
&\leq 2C_{p,S,A}R_{\max}\left|\frac{\gamma}{1-\gamma\alpha_{p,S,A}}-\frac{1}{1-\alpha_{p,S,A}}\right|\\
&\leq 2C_{p,S,A}R_{\max}\frac{1-\gamma}{(1-\alpha_{p,S,A})^2}.
\end{split}
\]

Noticing that $I_2(s)=\Expect_{a\sim \pi_{\theta^k}(\cdot|s)}[I_1(s,a)|\theta^k]$ and $\bar{V}^\pi=\Expect_{a\sim \pi_{\theta^k}(\cdot|s)}[\bar{Q}^\pi(s,a)|\theta^k]$, we see that the same bounds above for $I_1$ hold for $I_2$. More precisely, we have 
\[
\|I_2(s)\|_2\leq 2R_{\max}+\frac{2C_{p,S,A}R_{\max}}{1-\alpha_{p,S,A}},\quad \|I_2(s)-\bar{V}^{\pi_{\theta^k}}(s)\|_2\leq 2C_{p,S,A}R_{\max}\frac{1-\gamma}{(1-\alpha_{p,S,A})^2}.
\]

Hence we conclude that 
\[
\begin{split}
\|J_1-\nabla_\theta\eta(\pi_{\theta_k})\|_2&\leq \sum_{s\in\mathcal{S}}|w^{{\lfloor \beta H\rfloor}}(s;\pi_{\theta^k})-\mu_{\pi_{\theta^k}}(s)|\sum_{a\in\mathcal{A}}\pi_{\theta^k}(a|s)\|\nabla_{\theta}\log\pi_{\theta^k}(a|s)\|_2|I_1(s,a)-I_2(s)|\\
&\quad+\sum_{s\in\mathcal{S}}\mu_{\pi_{\theta^k}}(s)\sum_{a\in\mathcal{A}}\pi_{\theta^k}(a|s)\|\nabla_{\theta}\log\pi_{\theta^k}(a|s)\|_2\|I_1(s,a)-I_2(s)-\bar{A}^{\pi_{\theta^k}}(s,a)\|_2\\
&\leq \frac{8\tilde{C}R_{\max}C_{p,S,A}}{\lfloor \beta H\rfloor (1-\alpha_{p,S,A})}\left(1+\frac{C_{p,S,A}}{1-\alpha_{p,S,A}}\right)+4\tilde{C}R_{\max}C_{p,S,A}\frac{1-\gamma}{(1-\alpha_{p,S,A})^2}. 
\end{split}
\]

Similarly, for $J_2$, following the same analysis as above, we have that 
\[
\begin{split}
\|J_2\|_2&=\sum_{s\in\mathcal{S}}w^{\lfloor \beta H\rfloor}(s,\pi_{\theta^k})\sum_{a\in\mathcal{A}}\pi_{\theta^k}(a|s)\|\nabla_{\theta}\log\pi_{\theta^k}(a|s)\|_2\times \frac{\gamma^{(1-\beta)H}}{1-\gamma}|Q^{\gamma,\pi_{\theta^k}}(s,a)|\\
&\leq 2\tilde{C}R_{\max}\gamma^{(1-\beta)H}(1+C_{p,S,A}/(1-\alpha_{p,S,A})).
\end{split}
\]
Here we use the fact that $\frac{1}{1-\gamma}Q^{\gamma,\pi_{\theta^k}}(s,a)=I_1(s,a)$. 

Finally, combining the above bounds of $\|J_1-\nabla_{\theta}\eta(\pi_{\theta^k})$ and $\|J_2\|$, we obtain the desired result. 
\end{proof}

\subsection{Proof of Lemma \ref{spinup_stoc_grad_bds}}
\begin{proof}
We prove a slightly generalized version of the claimed results assuming only that $|r(s,a)|\leq R_{\max}$ 
instead of $r(s,a)\in[0,1]$ as in Assumption \ref{setting}. 

Firstly, by definition, we have 
\[
\|\hat{g}_k\|_2\leq \frac{2(R_{\max}+(1-\gamma)B)}{1-\gamma}+2\lambda, 
\]
where we use the fact that for the soft-max policy parametrization, $\|\nabla_{\theta}\log\pi_{\theta}(a|s)\|_2\leq 2$ for any $\theta\in\Theta$, $s\in\mathcal{S}$ and $a\in\mathcal{A}$ (\cf the proof of \cite[Lemma 2]{zhang2020sample}). 

Then by Lemma \ref{spinup_grad_error_lemma}, we have 
\[
\begin{split}
\Expect_k\hat{g}_k^T\nabla_{\theta}\bar{L}(\theta^k)&=(\Expect_k\hat{g}_k-\nabla_{\theta}\bar{L}(\theta^k))^T\nabla_{\theta}\bar{L}(\theta^k)+\|\nabla_{\theta}\bar{L}(\theta^k)\|_2^2\\
&\geq \|\nabla_{\theta}\bar{L}(\theta^k)\|_2^2 -(\bar{G}+2\lambda) \bar{\Delta}.
\end{split}
\]
Here $\bar{\Delta}$ is the right-hand side of \eqref{spinup_grad_error}, $\bar{G}=4R_{\max}\left(1+\frac{C_{p,S,A}}{1-\alpha_{p,S,A}}\right)$, 
and  we use the fact that 
\BEQ\label{grad_lbar_bd}
\begin{split}
\|\nabla_{\theta}\bar{L}(\theta^k)\|_2&\leq \|\nabla_{\theta}\bar{L}(\theta^k)\|_1\leq \sum_{s\in\mathcal{S},a\in\mathcal{A}}\mu_{\pi_{\theta^k}}(s)\pi_{\theta_k}(a|s)\left|\bar{A}^{\pi_{\theta}}(s,a)\right|+2\lambda \\
&\leq \max_{s\in\mathcal{S},a\in\mathcal{A}}\left|\bar{A}^{\pi_{\theta}}(s,a)\right|+2\lambda \leq 4R_{\max}\left(1+\frac{C_{p,S,A}}{1-\alpha_{p,S,A}}\right)+2\lambda.
\end{split}
\EEQ

Finally, by Lemma \ref{spinup_grad_error_lemma}, we have 
\[
\begin{split}
\|\Expect_k\hat{g}_k\|_2\leq \|\nabla_{\theta}\bar{L}(\theta^k)\|_2+\bar{\Delta},
\end{split}
\]
and hence 
\[
\begin{split}
\Expect_k\|\hat{g}_k\|_2^2&= \|\Expect_k\hat{g}_k\|_2^2+\text{\bf Var}_k\hat{g}_k\\
&= \|\Expect_k\hat{g}_k\|_2^2+\text{\bf Var}_k\hat{g}_k^1/N\\
&\leq 2\|\nabla_{\theta}\bar{L}(\theta^k)\|_2^2+2\bar{\Delta}^2+\Expect_k\|\hat{g}_k^1\|_2^2/N\\
&\leq 2\|\nabla_{\theta}\bar{L}(\theta^k)\|_2^2+2\bar{\Delta}^2 + (G^{\gamma}+2\lambda)^2/ N
\end{split}
\]
where $\hat{g}_k^1$ denotes the special case of $\hat{g}_k$ with $N=1$, $G^{\gamma}= \frac{2(R_{\max}+(1-\gamma)B)}{1-\gamma}$, and for a vector $X\in\reals^n$, $\text{\bf Var}_k X:=\sum_{i=1}^n\var_k X_i$ and $\var_k$ is the standard conditional variance given the $k$-th iteration $\theta_k$. Here we use the fact that $\text{\bf Var}_kX\leq \sum_{i=1}^n\Expect_k X_i^2=\Expect_k\|X\|_2^2$. This completes the proof. 
\end{proof}

\subsection{Proof of Theorem \ref{sample_complexity_eta_thm}}
\begin{proof}
By Lemma \ref{smooth_pg_ave} and an equivalent definition of strongly
smoothness (\cf \cite[Appendix]{ryu2016primer}), we have
\[
\begin{split}
-\bar L(\theta^{k+1})-(-\bar L(\theta^k))&\leq
-\nabla_{\theta}\bar L(\theta^k)^T(\theta^{k+1} -
\theta^k)+\dfrac{\bar \beta_{\lambda}}{2}\|\theta^{k+1}-\theta^k\|_2^2\\
&=\underbrace{-\alpha^k\nabla_{\theta}\bar L(\theta^k)^T
\hat{g}_k+\dfrac{\bar \beta_{\lambda}
(\alpha^k)^2}{2}\|\hat{g}_k\|_2^2}_{Y_k}.
\end{split}
\]

Let $Z_{k}= Y_k-\Expect_k[Y_k]$.
Then the above inequality implies that
\BEQ\label{k-step_ineq_ave}
\begin{split}
\bar L&(\theta^k) - \bar L(\theta^{k+1})\\
\leq& -\alpha^k
\nabla_{\theta} \bar L(\theta^k)^T\Expect_k\hat{g}_k+\dfrac{\bar \beta_{\lambda}
(\alpha^k)^2}{2}\Expect_k\|\hat{g}_k\|_2^2+Z_k\\
\leq& -\alpha^k\left(\|\nabla_{\theta}
\bar L(\theta^k)\|_2^2-(\bar G+2\lambda)\bar \Delta\right)+
\dfrac{\bar \beta_{\lambda}(\alpha^k)^2}{2}\left(\bar M+
2\|\nabla_{\theta}\bar L(\theta^k)\|_2^2\right)+Z_k\\
=&-\alpha^k(1-\bar \beta_{\lambda}\alpha^k)\|\nabla_{\theta}
\bar L(\theta^k)\|_2^2+\alpha^k(\bar G+2\lambda)\bar \Delta+
\dfrac{\bar \beta_{\lambda}\bar M(\alpha^k)^2}{2}+Z_k\\
\leq& -\dfrac{\alpha^k}{2}\|\nabla_{\theta}\bar L
(\theta^k)\|_2^2+\alpha^k(\bar G+2\lambda)\bar \Delta+
\dfrac{\bar \beta_{\lambda}\bar M(\alpha^k)^2}{2}+Z_k.
\end{split}
\EEQ
Here we use the fact that 
\[
\bar \beta_{\lambda}\alpha^k\leq \bar \beta_{\lambda}/(2\bar \beta_{\lambda})= 1/2.
\]

Now define $X_K=\sum_{k=0}^{K-1} Z_k$ (with $X_{l,0}=0$), then
\BEQ\label{cond_expect_mart_ave}
\Expect(X_{K+1}|\mathcal{F}_{K})=\sum_{k=0}^{K-1}Z_k+
\Expect(Y_K-\Expect_KY_K|\mathcal{F}_{K})=X_K.
\EEQ
Here $\mathcal{F}_{K}$ is the filtration up to episode $K$, 
\ie, the $\sigma$-algebra generated by all iterations 
$\{\theta^{0},\dots,\theta^{K}\}$ up to
the $K$-th one.
Notice that the second equality makes use of the fact that 
given the current policy,
the correspondingly 
sampled trajectory is conditionally independent of all previous policies and trajectories 
(as is always implicitly assumed in the literature of episodic reinforcement learning 
(\eg, \cf \cite{marbach2001simulation}).

In addition, for any $K\geq 1$,
\[
\begin{split}
|X_K-X_{K-1}|=&\,|Z_{K-1}|\leq \alpha^{K-1}\|\nabla_{\theta}
\bar L(\theta^{K-1})\|_2\|\Expect_{K-1}
\hat{g}_{K-1}-
\hat{g}_{K-1}\|_2\\
&+\dfrac{\bar \beta_{\lambda}(\alpha^{K-1})^2}{2}\left|\Expect_{K-1}
\|\hat{g}_{K-1}\|_2^2-
\|\hat{g}_{K-1}\|_2^2\right|\\
\leq&\,\underbrace{2(G^\gamma+2\lambda)\left(4\left(1+\dfrac{C_{p,S,A}}{1-\alpha_{p,S,A}}\right)+2\lambda\right)
\alpha^{K-1}+\dfrac{\bar \beta_{\lambda}}{2}
(G^\gamma+2\lambda)^2(\alpha^{K-1})^2}_{c_{K}}.
\end{split}
\]
Here we use the fact that
\[
\begin{split}
\|\nabla_{\theta}\bar{L}(\theta^k)\|_2
\leq 4\left(1+\frac{C_{p,S,A}}{1-\alpha_{p,S,A}}\right)+2\lambda,
\end{split}
\] which follows from \eqref{grad_lbar_bd} in the proof of Lemma \ref{spinup_stoc_grad_bds} with $R_{\max}=1$. 
The above inequality on $|X_K-X_{K-1}|$ 
also implies that
$\Expect|X_K|<\infty$, which, together with \eqref{cond_expect_mart_ave},
implies that $X_K$ is a martingale. 

Now by the definition of $\alpha^k$, it's easy to see that
$\sum_{K=1}^{\infty}c_{K}^2\leq \bar{C}<\infty$, where
\BEQ\label{C_l_ub_ave}
\bar{C}=\frac{8(G^\gamma+2\lambda)^2}{\bar \beta_{\lambda}^2}\left(2+\dfrac{2C_{p,S,A}}{1-\alpha_{p,S,A}}+\lambda\right)^2+
\dfrac{(G^\gamma+2\lambda)^4}{32\bar \beta_{\lambda}^2}.
\EEQ
Hence by 
Azuma-Hoeffding inequality,   
for any $c>0$ and $K\geq 0$, 
\BEQ\label{X_l_inf_bd_ave}
\prob(|X_K|\geq c)\leq 2e^{-c^2/(2\bar{C})}.
\EEQ

Then by summing up the inequalities \eqref{k-step_ineq_ave} from
$k=0$ to $K$, we obtain that
\BEQ
\begin{split}
\dfrac{1}{2}&\sum_{k=0}^{K}\alpha^k\|\nabla_{\theta}
\bar L(\theta^k)\|_2^2
\leq \sum_{k=0}^K\alpha^k(\bar G+2\lambda)\bar \Delta+\dfrac{\bar \beta_{\lambda}
\bar M\sum_{k=0}^\infty(\alpha^k)^2}{2}+\sum_{k=0}^{K}Z_k+
\sup_{\theta\in\Theta}\bar L(\theta)-\bar L(\theta^{0})\\
&\leq \sum_{k=0}^{K}\alpha^k(\bar G+2\lambda)\bar \Delta+
\dfrac{\bar \beta_{\lambda}\bar M}{2}\sum_{k=0}^{\infty}(\alpha^k)^2+
X_{K+1}+\eta^\star-\bar L(\theta^{0})\\
&\leq \underbrace{\frac{\bar M}{8\bar\beta_{\lambda}}
+\eta^\star-\bar L(\theta^{0})}_{\bar{D}}+X_{K+1}+(\bar G+2\lambda)\bar \Delta\sum_{k=0}^K\alpha^k,
\end{split}
\EEQ
where we use the fact that the regularization term $\Omega(\theta)\leq 0$
for all $\theta\in\Theta$.

Hence we have
\BEQ\label{weighted_grad_norm_ub_ave}
\begin{split}
\min_{k=0,\dots,K}\|\nabla_{\theta}\bar L(\theta^k)\|_2^2&\leq \dfrac{\sum_{k=0}^{K}\alpha^k\|\nabla_{\theta}\bar L(\theta^k)\|_2^2}{\sum_{k=0}^K\alpha^k}
\leq \frac{2(\bar{D}+|X_{K+1}|)}{\sum_{k=0}^K\alpha^k}+2(\bar G+2\lambda)\bar\Delta\\
&\leq 6\bar\beta_{\lambda}\frac{\bar D+|X_{K+1}|}{\sqrt{K+3}}\log_2(K+3)+2(\bar G+2\lambda)\bar\Delta,
\end{split}
\EEQ
where we use the fact that $\bar{D}\geq 0$.


Finally, by combining with the tail bound of \eqref{X_l_inf_bd_ave}, we conclude that for any $\epsilon>0$ and $\delta\in(0,1)$, 
for any 
\[
\begin{split}
K&\geq O\left(\dfrac{S^4A^4\bar\beta_{\lambda}^2(\bar D+\sqrt{2\bar C_\gamma\log(2/\delta)})^2}{\epsilon^4}\log^2 \left(\dfrac{SA\bar\beta_{\lambda}(\bar D+\sqrt{2\bar C_\gamma\log(2/\delta)})}{\epsilon}\right)\right),
\end{split}
\]
we have that with probability at least $1-\delta$, 
\[
\min_{k=0,\dots,K}\|\nabla_{\theta}\bar L(\theta^k)\|_2\leq \frac{\epsilon}{2SA}+\sqrt{2(\bar G+2\lambda)\bar \Delta}= \frac{\lambda}{2SA}
\]
and hence \eqref{sample_complexity_Vgamma} is satisfied as desired.  Here the last equality comes from noticing that our choice of $\lambda$ is a root of the following quadratic equation:
\[
2(\bar G+2\lambda)\bar \Delta=\frac{(\lambda-\epsilon)^2}{4S^2A^2}.
\]

Here since  $\bar \beta_{\lambda}\geq 8$, $\eta(\pi)\in[0,1]$, we have 
\[
\bar{D}=O(\bar M+\lambda+1),\quad \bar{C}=\,O\left(\frac{(G^{\gamma}+2\lambda)^2}{S}\left(\frac{C_{p,S,A}^2}{(1-\alpha_{p,S,A})^2}+\lambda^2+(G^{\gamma}+2\lambda)^2\right)\right.
\]
where the constants hidden in the big-$O$ notation may depend on $\theta^0$ (and the constants $B$ and $\beta$). 
\end{proof}





\subsection{Proof of Theorem \ref{main_res_finite_average}}
\begin{proof}[Proof of Theorem \ref{main_res_finite_average}]
By Lemmas \ref{VH-eta-pi} and \ref{VH-eta} and the fact that $\eta^\star-\eta(\hat{\pi})\leq \epsilon$, we have
\[
\begin{split}
V^{H,\star}-V^H(\hat{\pi})&\leq |V^{H,\star}-\eta^\star|+|\eta^\star-\eta(\hat{\pi})|+|\eta(\hat{\pi})-V^H(\hat{\pi})|\\
&\leq \frac{2R_{\max}D_{p,S,A}}{H}+\epsilon+\frac{2R_{\max}C_{p,S,A}}{H(1-\alpha_{p,S,A})}.
\end{split}
\]
This completes our proof. 
\end{proof}

\subsection{A more detailed statement of Theorem \ref{main_conv_spin-up}}\label{formal_statement_dae_reinforce}
In this section, we provide a more detailed statement of Theorem \ref{main_conv_spin-up}, which displays the dependencies of the constants on the problem and algorithm parameters in a more explicit manner and provides a slightly tighter sub-optimality bound in terms of the (non-dominating) constants. 
\begin{theorem}\label{spinup_thm}
Given Assumptions \ref{finite-ergodic} and \ref{setting}, let $\gamma=1-H^{-\sigma}$ for some $\sigma\in(0,1)$. For  any $\epsilon>0$ and $\delta\in(0,1)$, set $\lambda$, $\bar{\beta}_\lambda$ and $\alpha^k$ to be the same as in Theorem \ref{sample_complexity_eta_thm}. 
Then for any $K$ such that \eqref{K_lower_bound_eta} is satisfied with 
{\small
\[
\begin{split}
\bar{C}&=\frac{8(G^\gamma+2\lambda)^2}{\bar \beta_{\lambda}^2}\left(2+\dfrac{2C_{p,S,A}}{1-\alpha_{p,S,A}}+\lambda\right)^2+
\dfrac{(G^\gamma+2\lambda)^4}{32\bar \beta_{\lambda}^2}\\
&=O\left(\frac{H^{4\sigma}+\bar{E}^2H^{2\sigma}+S^2A^2(S^2A^2+\bar{E})\left(\frac{\bar{E}^4}{\beta H}+\frac{\bar{E}^4}{H^{\sigma}}\right)+S^4A^4(S^4A^4+\bar{E}^2)\left(\frac{\bar{E}^4}{\beta^2 H^2}+\frac{\bar{E}^4}{H^{2\sigma}}\right)}{S}\right),\\
\bar{D}&=\frac{\bar M}{8\bar\beta_{\lambda}}
+\eta^\star-\bar L(\theta^{0})\\
&=O\left(\frac{\bar{E}^4}{\beta^2H^2}+\frac{\bar{E}^4}{H^{2\sigma}}+\frac{H^{2\sigma}+S^2A^2(S^2A^2+\bar{E})\left(\frac{\bar{E}^2}{\beta H}+\frac{\bar{E}^2}{H^{\sigma}}\right)}{N}+SA(SA+\sqrt{\bar{E}})\left(\frac{\bar{E}}{\sqrt{\beta H}}+\frac{\bar{E}}{H^{\frac{\sigma}{2}}}\right)+1\right),
\end{split}
\] 
}
where $\bar{E}=C_{p,S,A}/(1-\alpha_{p,S,A})$, 
with probability at least $1-\delta$, we have 
\BEQ\label{sample_complexity_VH_spinup_appendix}
\begin{split}
\min_{k=0,\dots,K}V^{H,\star}-V^{H}(\pi_{\theta^k})\leq&\, O\left(\frac{S}{1-\alpha_{p,S,A}}
\left(\epsilon+\frac{S^2A^2C_{p,S,A}^2}{(1-\alpha_{p,S,A})^2}((\beta H)^{-1}+H^{-\sigma})\right.\right.\\
&\quad +\left.\frac{SAC_{p,S,A}^3}{(1-\alpha_{p,S,A})^3}((\beta H)^{-1/2}+H^{-\sigma/2})\right)\\
&\left.\quad +\left(D_{p,S,A}+\frac{C_{p,S,A}}{1-\alpha_{p,S,A}}\right)H^{-1}\right).
\end{split}
\EEQ
Here $G^{\gamma}$ and $\bar{M}$ are the constants defined in Lemma \ref{spinup_stoc_grad_bds}, while 
$C_{p,S,A}> 1$, $D_{p,S,A}> 1$ and $\alpha_{p,S,A}\in[0,1)$ are the constants in Proposition \ref{dobrushin_uniform} and Lemma \ref{VH-eta}.
\end{theorem}
\begin{proof}
The key is to notice that we have 
\[
\begin{split}
\bar{\Delta} &=O\left(\frac{C_{p,S,A}^2}{\beta(1-\alpha_{p,S,A})^2 H}+\frac{C_{p,S,A}}{(1-\alpha_{p,S,A})^2}H^{-\sigma}+\frac{C_{p,S,A}}{1-\alpha_{p,S,A}}e^{-(1-\beta)H^{1-\sigma}}\right)\\
&=O\left(\frac{C_{p,S,A}^2}{(1-\alpha_{p,S,A})^2}\left(\frac{1}{\beta H}+H^{-\sigma}\right)\right).
\end{split}
\]
The proof then follows by plugging in the constants and elementary simplifications, and is hence omitted. 
\end{proof}

\section{Proofs for Doubly Discounted REINFORCE algorithm}
In this section we provide the proofs for the convergence result of Doubly Discounted REINFORCE algorithm.

Proof of Lemma \ref{stoc_grad_bds} is a direct implication of \cite[Lemmas 2 and 12]{zhang2020sample} and \cite[Lemma B.1]{liu2020improved}. We omit the details here. Below we provide the proof of Theorem \ref{sample_complexity_Vgamma_thm} and a more detailed statement of Theorem \ref{discounted_reinforce_thm}.
\subsection{Proof of Theorem \ref{sample_complexity_Vgamma_thm}}
The proof of Theorem \ref{sample_complexity_Vgamma_thm} follows similar steps as in Theorem \ref{sample_complexity_eta_thm}. But for self-containedness, we still include the complete proof below. 
\begin{proof}
By Proposition \ref{Lsmooth} and an equivalent definition of strongly
smoothness (\cf \cite[Appendix]{ryu2016primer}), we have
\[
\begin{split}
-L^\gamma(\theta^{k+1})-(-L^\gamma(\theta^k))&\leq
-\nabla_{\theta}L^\gamma(\theta^k)^T(\theta^{k+1} -
\theta^k)+\dfrac{\beta_{\lambda}}{2}\|\theta^{k+1}-\theta^k\|_2^2\\
&=\underbrace{-\alpha^k\nabla_{\theta}L^\gamma(\theta^k)^T
\tilde{g}_k+\dfrac{\beta_{\lambda}
(\alpha^k)^2}{2}\|\tilde{g}_k\|_2^2}_{Y_k}.
\end{split}
\]

Let $Z_{k}= Y_k-\Expect_k[Y_k]$.
Then the above inequality implies that
\BEQ\label{k-step_ineq}
\begin{split}
L^\gamma&(\theta^k) - L^\gamma(\theta^{k+1})\\
\leq& -\alpha^k
\nabla_{\theta} L^\gamma(\theta^k)^T\Expect_k\tilde{g}_k+\dfrac{\beta_{\lambda}
(\alpha^k)^2}{2}\Expect_k\|\tilde{g}_k\|_2^2+Z_k\\
\leq& -\alpha^k\left(\|\nabla_{\theta}
L^\gamma(\theta^k)\|_2^2-(G+2\lambda)\Delta\right)+
\dfrac{\beta_{\lambda}(\alpha^k)^2}{2}\left(M+
2\|\nabla_{\theta}L^\gamma(\theta^k)\|_2^2\right)+Z_k\\
=&-\alpha^k(1-\beta_{\lambda}\alpha^k)\|\nabla_{\theta}
L^\gamma(\theta^k)\|_2^2+\alpha^k(G+2\lambda)\Delta+
\dfrac{\beta_{\lambda}M(\alpha^k)^2}{2}+Z_k\\
\leq& -\dfrac{\alpha^k}{2}\|\nabla_{\theta}L^\gamma
(\theta^k)\|_2^2+\alpha^k(G+2\lambda)\Delta+
\dfrac{\beta_{\lambda}M(\alpha^k)^2}{2}+Z_k.
\end{split}
\EEQ
Here we use the fact that 
\[
\beta_{\lambda}\alpha^k\leq \beta_{\lambda}/(2\beta_{\lambda})= 1/2.
\]

Now define $X_K=\sum_{k=0}^{K-1} Z_k$ (with $X_{l,0}=0$), then
\BEQ\label{cond_expect_mart}
\Expect(X_{K+1}|\mathcal{F}_{K})=\sum_{k=0}^{K-1}Z_k+
\Expect(Y_K-\Expect_KY_K|\mathcal{F}_{K})=X_K.
\EEQ
Here $\mathcal{F}_{K}$ is the filtration up to episode $K$, 
\ie, the $\sigma$-algebra generated by all iterations 
$\{\theta^{0},\dots,\theta^{K}\}$ up to
the $K$-th one.
Notice that the second equality makes use of the fact that 
given the current policy,
the correspondingly 
sampled trajectory is conditionally independent of all previous policies and trajectories.

In addition, for any $K\geq 1$,
\[
\begin{split}
|X_K-X_{K-1}|=&\,|Z_{K-1}|\leq \alpha^{K-1}\|\nabla_{\theta}
L^{\gamma}(\theta^{K-1})\|_2\|\Expect_{K-1}
\tilde{g}_{K-1}-
\tilde{g}_{K-1}\|_2\\
&+\dfrac{\beta_{\lambda}(\alpha^{K-1})^2}{2}\left|\Expect_{K-1}
\|\tilde{g}_{K-1}\|_2^2-
\|\tilde{g}_{K-1}\|_2^2\right|\\
\leq&\,\underbrace{2(G+2\lambda)\left(\dfrac{2}{(1-\gamma)^2}+2\lambda\right)
\alpha^{K-1}+\dfrac{\beta_{\lambda}}{2}
(G+2\lambda)^2(\alpha^{K-1})^2}_{c_{K}}.
\end{split}
\]
Here we use the fact that
\[
\|\nabla_{\theta}L^\gamma(\theta^{K-1})\|_2
\leq 2/(1-\gamma)^2+2\lambda,
\] which follows from \eqref{vgamma_grad} and \eqref{omega_grad} 
similarly as in \eqref{grad_lbar_bd}.
The above inequality on $|X_K-X_{K-1}|$ 
also implies that
$\Expect|X_K|<\infty$, which, together with \eqref{cond_expect_mart},
implies that $X_K$ is a martingale. 

Now by the definition of $\alpha^k$, it's easy to see that
$\sum_{K=1}^{\infty}c_{K}^2\leq C<\infty$, where
\BEQ\label{C_l_ub}
C=\frac{8(G+2\lambda)^2}{\beta_{\lambda}^2}\left(\dfrac{1}{(1-\gamma)^2}
+\lambda\right)^2+
\dfrac{(G+2\lambda)^4}{32\beta_{\lambda}^2}.
\EEQ
Hence by 
Azuma-Hoeffding inequality,   
for any $c>0$ and $K\geq 0$, 
\BEQ\label{X_l_inf_bd}
\prob(|X_K|\geq c)\leq 2e^{-c^2/(2C)}.
\EEQ

Then by summing up the inequalities \eqref{k-step_ineq} from
$k=0$ to $K$, we obtain that
\BEQ
\begin{split}
\dfrac{1}{2}&\sum_{k=0}^{K}\alpha^k\|\nabla_{\theta}
L^\gamma(\theta^k)\|_2^2
\leq \sum_{k=0}^K\alpha^k(G+2\lambda)\Delta+\dfrac{\beta_{\lambda}
M\sum_{k=0}^\infty(\alpha^k)^2}{2}+\sum_{k=0}^{K}Z_k+
\sup_{\theta\in\Theta}L^\gamma(\theta)-L^\gamma(\theta^{0})\\
&\leq \sum_{k=0}^{K}\alpha^k(G+2\lambda)\Delta+
\dfrac{\beta_{\lambda}M}{2}\sum_{k=0}^{\infty}(\alpha^k)^2+
X_{K+1}+V^{\gamma,\star}/(1-\gamma)-L^\gamma(\theta^{0})\\
&\leq \underbrace{\frac{M}{8\beta_{\lambda}}
+V^{\gamma,\star}/(1-\gamma)-L^\gamma(\theta^{0})}_{D}+X_{K+1}+(G+2\lambda)\Delta\sum_{k=0}^K\alpha^k,
\end{split}
\EEQ
where we use the fact that the regularization term $\Omega(\theta)\leq 0$
for all $\theta\in\Theta$.

Hence we have
\BEQ\label{weighted_grad_norm_ub}
\begin{split}
\min_{k=0,\dots,K}\|\nabla_{\theta}L^\gamma(\theta^k)\|_2^2&\leq \dfrac{\sum_{k=0}^{K}\alpha^k\|\nabla_{\theta}L^\gamma(\theta^k)\|_2^2}{\sum_{k=0}^K\alpha^k}
\leq \frac{2(D+|X_{K+1}|)}{\sum_{k=0}^K\alpha^k}+2(G+2\lambda)\Delta\\
&\leq 6\beta_{\lambda}\frac{D+|X_{K+1}|}{\sqrt{K+3}}\log_2(K+3)+2(G+2\lambda)\Delta,
\end{split}
\EEQ
where we use the fact that $D\geq 0$. 

Finally, by combining with the tail bound of \eqref{X_l_inf_bd}, we conclude that for any $\epsilon>0$ and $\delta\in(0,1)$, 
for any 
\[
\begin{split}
K&\geq O\left(\dfrac{S^4A^4\beta_{\lambda}^2(D+\sqrt{2C\log(2/\delta)})^2}{\epsilon^4}\log^2 \left(\dfrac{SA\beta_{\lambda}(D+\sqrt{2C\log(2/\delta)})}{\epsilon}\right)\right),
\end{split}
\]
we have that with probability at least $1-\delta$, 
\[
\min_{k=0,\dots,K}\|\nabla_{\theta}L^\gamma(\theta^k)\|_2\leq \frac{\epsilon}{2SA}+\sqrt{2(G+2\lambda)\Delta}= \frac{\lambda}{2SA}
\]
and hence \eqref{sample_complexity_Vgamma} is satisfied as desired.  Here the last equality comes from noticing that our choice of $\lambda$ is a root of the following quadratic equation:
\[
2(G+2\lambda)\Delta=\frac{(\lambda-\epsilon)^2}{4S^2A^2}.
\]

Here since  $\beta_{\lambda}\geq 8$, $V^{\gamma}(\pi)\in[0,1]$, we have 
\[
D=O(M+1/(1-\gamma)+\lambda),\quad C=O((G+2\lambda)^2(1/(1-\gamma)^4+\lambda^2+(G+2\lambda)^2)).
\]
where the constants hidden in the big-$O$ notation may depend on $\theta^0$ (and the constant $B$). 
\end{proof}



\subsection{Proof of Theorem \ref{main_res_finite_discount}}
\begin{proof}[Proof of Theorem \ref{main_res_finite_discount}]
Notice that we have 
\[
V^{H,\star}-V^H(\hat{\pi})\leq |V^{H,\star}-\eta^\star|+|\eta^\star-V^{\gamma,\star}|+|V^{\gamma,\star}-V^\gamma(\hat{\pi})|+|V^{\gamma}(\hat{\pi})-V^H(\hat{\pi})|.
\]
Hence combining Lemma \ref{error_bd_arb}, Corollary \ref{Vgamma-eta} and Lemma \ref{VH-eta} and by the fact that $V^{\gamma,\star}-V^{\gamma}(\hat{\pi})\leq \epsilon$, the proof is complete. 
\end{proof}

\subsection{A more detailed statement of Theorem \ref{discounted_reinforce_thm}}\label{formal_statement_dd_reinforce}
Similarly, in this section, we provide a more detailed statement of Theorem \ref{discounted_reinforce_thm}, which displays the dependencies of the constants on the problem and algorithm parameters in a more explicit manner and provides a slightly tighter sub-optimality bound in terms of the (non-dominating) constants. 
\begin{theorem}\label{discounted_reinforce_cor}
Given Assumptions \ref{finite-ergodic} and \ref{setting}, let $\gamma=1-H^{-\sigma}$ for some $\sigma\in(0,1)$. For any $\epsilon>0$, $\delta\in(0,1)$, set $\lambda$, $\beta_\lambda$ and $\alpha^k$ to be the same as in Theorem \ref{sample_complexity_Vgamma_thm}. 
Then for any $K$ such that \eqref{K_lower_bound} is satisfied with 
\[
\begin{split}
C&=\frac{8(G+2\lambda)^2}{\beta_{\lambda}^2}\left(\dfrac{1}{(1-\gamma)^2}
+\lambda\right)^2+
\dfrac{(G+2\lambda)^4}{32\beta_{\lambda}^2}=O((H^{4\sigma}+S^4A^4H^{1+3\sigma}e^{-H^{1-\sigma}})^2), \\
D&=\frac{M}{8\beta_{\lambda}}
+V^{\gamma,\star}/(1-\gamma)-L^\gamma(\theta^{0})\\
&=O\left(S^2A^2H^{\frac{1+3\sigma}{2}}e^{-\frac{H^{1-\sigma}}{2}}+H^{\sigma}+\frac{S^4A^4H^{1+3\sigma}e^{-H^{1-\sigma}}+H^{4\sigma}}{N}\right),
\end{split}
\] 
with probability at least $1-\delta$, 
\BEQ\label{sample_complexity_VH_appendix}
\begin{split}
\min_{k=0,\dots,K}V^{H,\star}-V^{H}(\pi_{\theta^k})\leq&\, O\left(\min\left\{\left\|\frac{1}{\rho}\right\|_{\infty},\frac{S}{1-\alpha_{p,S,A}}\right\}(\epsilon+S^2A^2H^{\frac{1+3\sigma}{2}}e^{-H^{1-\sigma}/2})\right.\\ 
&\quad+ C_{p,S,A}\left(\frac{1}{H^{1-\sigma}}-\frac{1}{H}\right)\alpha_{p,S,A}^H\\
&\quad \left.+\dfrac{1}{H}\left(\dfrac{C_{p,S,A}(H^{1-\sigma}+\alpha_{p,S,A})}{1-\alpha_{p,S,A}}+D_{p,S,A}\right)\right). 
\end{split}
\EEQ
Here $G$ and $M$ are constants defined in Lemma \ref{stoc_grad_bds}, 
while $C_{p,S,A}> 1$, $D_{p,S,A}> 1$ and $\alpha_{p,S,A}\in[0,1)$ are the constants in Proposition \ref{dobrushin_uniform} and Lemma \ref{VH-eta}.
\end{theorem}
\begin{proof}
The key is to notice that we have 
\[
\Delta =2H^{\sigma}(1-H^{-\sigma})^H(H+H^{\sigma})\leq 4H^{1+\sigma}((1-1/H^{\sigma})^{H^{\sigma}})^{H^{1-\sigma}}\leq 4H^{1+\sigma}e^{-H^{1-\sigma}}
\]
The proof then follows by plugging in the constants and elementary simplifications, and is hence omitted. 
\end{proof}
\begin{remark}
Note that from the slightly more refined bound above, we see that compared with the bias term in DAE REINFORCE, additional constant improvements in the exponential term can be achieved when $\min_{s\in\mathcal{S}}\rho(s)$ is relatively large (\eg, when it is lower bounded by $\frac{1-\alpha_{p,S,A}}{S}$). 
\end{remark}

\end{document}